\documentclass{article}

\usepackage{microtype}
\usepackage{graphicx}
\usepackage{subfigure}
\usepackage{enumerate}
\usepackage{enumitem}
\usepackage{booktabs} 
\usepackage{notation}
\usepackage{hyperref}

\usepackage{xcolor}

\newcommand{\algname}{\text{NeuralUCB}}
\newcommand{\ifcomments}{\iftrue}

\def \CC {}
\def \NEW {}

\newcommand{\dsfont}[1]{\texttt{#1}}
\def \la {\langle}
\def \ra {\rangle}

\usepackage[accepted]{icml2020}

\icmltitlerunning{Neural Contextual Bandits with UCB-based Exploration}

\begin{document}

\twocolumn[
\icmltitle{Neural Contextual Bandits with UCB-based Exploration}



\icmlsetsymbol{equal}{*}

\begin{icmlauthorlist}
\icmlauthor{Dongruo Zhou}{ucla}
\icmlauthor{Lihong Li}{google}
\icmlauthor{Quanquan Gu}{ucla}
\end{icmlauthorlist}

\icmlaffiliation{ucla}{Department of Computer Science, University of California, Los Angeles, CA 90095, USA}
\icmlaffiliation{google}{Google Research, USA}

\icmlcorrespondingauthor{Quanquan Gu}{qgu@cs.ucla.edu}

\icmlkeywords{Machine Learning, ICML}

\vskip 0.3in
]



\printAffiliationsAndNotice{} 

\begin{abstract}
We study the stochastic contextual bandit problem, where the reward is generated from an unknown function with additive noise.  No assumption is made about the reward function other than boundedness.  We propose a new algorithm, NeuralUCB, which leverages the representation power of deep neural networks and uses a neural network-based random feature mapping to construct an upper confidence bound (UCB) of reward for efficient exploration. We prove that, under standard assumptions, NeuralUCB achieves $\tilde O(\sqrt{T})$ regret, 
where $T$ is the number of rounds. To the best of our knowledge, it is the first neural network-based contextual bandit algorithm with a near-optimal regret guarantee.
We also show the algorithm is empirically competitive against representative baselines in a number of benchmarks.
\end{abstract}

\section{Introduction}
The stochastic contextual bandit problem has been extensively studied in machine learning~\citep{langford08epoch,bubeck12regret,lattimore19bandit}: 
at round $t \in \{1,2,\ldots, T\}$, an agent is presented with a set of $K$
actions, each of which is associated with a $d$-dimensional feature vector. 
After choosing an action, the agent will receive a stochastic reward generated from some unknown distribution conditioned on 
the action's feature vector. The goal of the agent is to maximize the expected cumulative rewards over $T$ rounds. 
Contextual bandit algorithms have been applied to many real-world applications, such as personalized recommendation, advertising and Web search. 

The most studied model 
in the literature is linear contextual bandits~\citep{auer2002using,abe2003reinforcement,dani2008stochastic, rusmevichientong2010linearly}, 
which assumes that the expected reward at each round is linear in the feature vector. 
While successful in both theory and practice~\citep{li2010contextual,chu2011contextual,abbasi2011improved}, the linear-reward assumption it makes
often fails to hold in practice, which motivates the study of nonlinear or nonparametric contextual bandits~\citep{filippi2010parametric,srinivas2009gaussian,bubeck2011x,valko2013finite}.
However, they still require fairly restrictive assumptions on the reward function.  For instance, \citet{filippi2010parametric} make a generalized linear model assumption on the reward, \citet{bubeck2011x} require it to have a Lipschitz continuous property in a proper metric space, and \citet{valko2013finite} assume the reward function belongs to some Reproducing Kernel Hilbert Space (RKHS).

In order to overcome the above shortcomings, 
deep neural networks (DNNs) ~\citep{goodfellow2016deep} have been introduced to learn the underlying reward function in contextual bandit problem, thanks to their strong representation power.  
We call these approaches collectively as \emph{neural contextual bandit algorithms}.
Given the fact that DNNs enable the agent to 
make use of nonlinear models with less domain knowledge, existing work \citep{riquelme2018deep,zahavy2019deep} study \emph{neural-linear bandits}. 
That is, they use 
all but the last layers of a DNN as a feature map, which transforms contexts from the raw input space to a low-dimensional space, usually with better representation and less frequent updates. Then they learn a linear exploration
policy on top of the last hidden layer of the DNN with more frequent updates. These attempts have achieved great empirical success, but no regret guarantees are provided.

In this paper, we consider 
\emph{provably efficient} neural contextual bandit algorithms.  The new algorithm, \algname, uses a neural network 
to learn the unknown reward function, and follows a UCB strategy for exploration.  At the core of the algorithm is the novel use of DNN-based random feature mappings to construct the UCB.  Its regret analysis 
is built on recent advances on optimization and generalization of deep neural networks~\citep{jacot2018neural, arora2019exact, cao2019generalization2}. 
Crucially, the analysis makes no modeling assumptions about the reward function, other than that it be bounded. 
While the main focus of our paper is theoretical, we also show in a few benchmark problems the effectiveness of \algname, and demonstrate its benefits against several representative baselines.

\newcommand{\reducevspace}{\vspace{-2mm}}
Our main contributions are as follows:
\reducevspace
\begin{itemize}[leftmargin = *]
    \item We propose a neural contextual bandit algorithm 
    that can be regarded as an extension of existing (generalized) linear bandit algorithms~\citep{abbasi2011improved,filippi2010parametric,li2010contextual,li2017provably} to the case of arbitrary bounded reward functions.
    \reducevspace
    \item We prove that, under 
    standard assumptions, our algorithm is able to achieve $\tilde O(\tilde d \sqrt{T})$ regret,  where $\tilde d$ is the effective dimension of a neural tangent kernel matrix and $T$ is the number of rounds. The bound recovers the existing $\tilde O( d \sqrt{T})$ regret for linear contextual bandit as a special case \citep{abbasi2011improved}, where $d$ is the dimension of context.
    \reducevspace
    \item We demonstrate empirically the effectiveness of the algorithm in both synthetic and benchmark problems.
\end{itemize}
\reducevspace

\paragraph{Notation:} Scalars are denoted by lower case letters, vectors by lower case bold face letters, and matrices by upper case bold face letters. For a positive integer $k$, $[k]$ denotes $\{1,\dots,k\}$. For a vector $\btheta \in \RR^d$, we denote its $\ell_2$ norm by $\| \btheta \|_2 = \sqrt{\sum_{i=1}^d \theta_i^2}$ and its $j$-th coordinate by $[\btheta]_j$.  
For a matrix $\Ab \in \RR^{d \times d}$, we denote its spectral norm, Frobenius norm, and $(i,j)$-th entry by $\|\Ab\|_2$, $\|\Ab\|_F$, and $[\Ab]_{i,j}$, respectively.
We denote a sequence of vectors by $\{\btheta_j\}_{j=1}^t$, and similarly for matrices. 
For two sequences $\{a_n\}$ and $\{b_n\}$, we use $a_n = O(b_n)$ to denote that there exists some constant $C>0$ such that $a_n \leq Cb_n$; similarly, $a_n = \Omega(b_n)$ means there exists some constant $C'>0$ such that $a_n \geq C'b_n$. In addition, we use $\tilde O(\cdot)$ to hide logarithmic factors. We say a random variable $X$ is $\nu$-sub-Gaussian if $\EE \exp(\lambda(X - \EE X)) \leq\exp(\lambda^2\nu^2/2)$ for any $\lambda>0$.

\section{Problem Setting}
We consider the stochastic $K$-armed contextual bandit problem, where the total number of rounds $T$ is known. At round $t \in [T]$, the agent observes the context consisting of $K$ feature vectors: $\{\xb_{t,a} \in \RR^d~|~a \in [K]\}$. The agent selects an action $a_t$ 
and receives a reward $r_{t,a_t}$. For brevity, we denote by $\{\xb^i\}_{i=1}^{TK}$ the collection of $\{\xb_{1,1}, \xb_{1,2},\dots,\xb_{T,K}\}$. 
Our goal is to maximize the following \emph{pseudo regret} (or \emph{regret} for short):
\begin{align}
    R_T = \EE\bigg[\sum_{t=1}^T (r_{t,a_t^*} - r_{t,a_t})\bigg],\label{def:regret}
\end{align}
where $a_t^* = \argmax_{a \in[K]}\EE [r_{t,a}]$ is the optimal action at round $t$ that maximizes the expected reward. 

This work makes the following assumption about reward generation: for any round $t$,
\begin{align}
    r_{t,a_t} = h(\xb_{t,a_t}) + \xi_t,\label{eq:reward function}
\end{align}
where $h$ is an unknown function satisfying $0 \leq h(\xb) \leq 1$ for any $\xb$, and $\xi_t$ is $\nu$-sub-Gaussian noise conditioned on $\xb_{1,a_1},\dots,\xb_{t-1,a_{t-1}}$ satisfying $\EE \xi_t = 0$. The $\nu$-sub-Gaussian assumption for $\xi_t$ is standard in the stochastic bandit literature~\citep[e.g.,][]{abbasi2011improved,li2017provably}, and is satisfied by, for example, any bounded noise. 
The bounded $h$ assumption holds true when $h$ belongs to linear functions, generalized linear functions, Gaussian processes, and kernel functions with bounded RKHS norm over a bounded domain, among others.

In order to learn the reward function $h$ in \eqref{eq:reward function}, we propose to use a fully connected neural networks with depth $L \geq 2$:
\begin{align}
    &f(\xb; \btheta) = \sqrt{m}\Wb_L \sigma\Big(\Wb_{L-1}\sigma\big(\cdots \sigma (\Wb_1\xb)\big)\Big), \label{def:network}
\end{align}
where $\sigma(x) = \max\{x,0\}$ is the rectified linear unit (ReLU) activation function, $\Wb_1 \in \RR^{m \times d}, \Wb_i \in \RR^{m \times m}, 2 \leq i \leq L-1, \Wb_L \in \RR^{m \times 1}$, and $\btheta = [\text{vec}(\Wb_1)^\top,\dots,\text{vec}(\Wb_L)^\top]^\top \in \RR^{p}$ with $p = m+md+m^2 (L-1)$. 
Without loss of generality, we assume that the width of each hidden layer is the same (i.e., $m$) for convenience in analysis. 
We denote the gradient of the neural network function by $\gb(\xb; \btheta) = \nabla_{\btheta} f(\xb; \btheta) \in \RR^p$.

\section{The \algname\ Algorithm}\label{sec:algorithm}

The key idea of \algname\ (Algorithm~\ref{algorithm:2}) is to use a neural network $f(\xb; \btheta)$ to predict the reward of context $\xb$, 
and upper confidence bounds computed from the network to guide exploration \citep{auer2002using}.

\paragraph{Initialization} It initializes the network 
by randomly generating each entry of $\btheta$ from an appropriate Gaussian distribution: for $1 \leq l\leq L-1$, $\Wb_l$ is set to be $\begin{pmatrix}
    \Wb & \zero \\
    \zero & \Wb
    \end{pmatrix}$,
where each entry of $\Wb$ is generated independently from $N(0, 4/m)$; 
$\Wb_L$ is set to $(\wb^\top, -\wb^\top)$, where each entry of $\wb$ is generated independently from $N(0, 2/m)$.

\paragraph{Learning}
At round $t$, Algorithm \ref{algorithm:2} 
observes the contexts for all actions, $\{\xb_{t,a}\}_{a=1}^K$.  First, it computes an upper confidence bound $U_{t,a}$ for each action $a$, based on $\xb_{t,a}$, $\theta_{t-1}$ (the current neural network parameter), and a positive scaling factor $\gamma_{t-1}$.  It then chooses action $a_t$ with the largest $U_{t,a}$,
and receives the corresponding reward $r_{t,a_t}$.  At the end of round $t$, \algname\ updates $\btheta_t$ by applying Algorithm~\ref{algorithm:GD} to  (approximately) minimize
$L(\btheta)$ using gradient descent, 
and updates $\gamma_t$. 
We choose gradient descent in Algorithm~\ref{algorithm:GD} for the simplicity of analysis, although the training method can be replaced by stochastic gradient descent with a more involved analysis~\citep{allen2018convergence, zou2018stochastic}.

\begin{algorithm*}[h]
	\caption{\algname}\label{algorithm:2}
	\begin{algorithmic}[1]
	\STATE \textbf{Input:} Number of rounds $T$, regularization parameter $\lambda$, exploration parameter $\nu$, confidence parameter $\delta$, norm parameter $S$, step size $\eta$, number of gradient descent steps $J$, network width $m$, network depth $L$.
	\STATE \textbf{Initialization:} Randomly initialize $\btheta_0$ as described in the text
	\STATE Initialize $\Zb_0 =  \lambda\Ib$
    \FOR{$t = 1,\dots,T$}\STATE Observe $\{\xb_{t,a}\}_{a=1}^K$
    \FOR{$a = 1,\dots,K$}
    \STATE Compute 
        $U_{t,a} = f(\xb_{t,a}; \btheta_{t-1})  + \gamma_{t-1}\sqrt{\gb(\xb_{t,a}; \btheta_{t-1})^\top\Zb_{t-1}^{-1}\gb(\xb_{t,a}; \btheta_{t-1})/m}$
    \STATE Let $a_t =\argmax_{a \in[K]}U_{t,a}$
    \ENDFOR
    \STATE Play $a_t$ and observe reward $r_{t,a_t}$
    \STATE Compute $\Zb_t = \Zb_{t-1} + \gb(\xb_{t,a_t}; \btheta_{t-1})\gb(\xb_{t,a_t}; \btheta_{t-1})^\top/m$ 
    \STATE Let $\btheta_t =  \text{TrainNN}(\lambda, \eta, J, m, \{\xb_{i, a_i}\}_{i=1}^t, \{r_{i, a_i}\}_{i=1}^t, \btheta_0)$
    \STATE Compute 
    \begin{align}
        \gamma_t &= \sqrt{1+C_1m^{-1/6}\sqrt{\log m}L^4t^{7/6}\lambda^{-7/6}} \cdot \bigg(\nu\sqrt{\log\frac{\det \Zb_t}{\det \lambda \Ib} + C_2m^{-1/6}\sqrt{\log m} L^4t^{5/3}\lambda^{-1/6}-2\log\delta} + \sqrt{\lambda}S\bigg)\notag \\
    &\qquad + (\lambda + C_3 tL)\Big[(1- \eta m \lambda)^{J/2} \sqrt{t/\lambda} + m^{-1/6}\sqrt{\log m}L^{7/2}t^{5/3}\lambda^{-5/3}(1+\sqrt{t/\lambda})\Big].\notag
    \end{align}
        \ENDFOR
	\end{algorithmic}
\end{algorithm*}
\vspace{5mm}
\begin{algorithm}
	\caption{TrainNN($\lambda, \eta, U, m, \{\xb_{i, a_i}\}_{i=1}^t, \{r_{i, a_i}\}_{i=1}^t, \btheta^{(0)}$)} \label{algorithm:GD} 
	\begin{algorithmic}[1]
	\STATE \textbf{Input:} Regularization parameter $\lambda$, step size $\eta$, number of gradient descent steps $U$, network width $m$, contexts $\{\xb_{i, a_i}\}_{i=1}^t$, rewards $\{r_{i, a_i}\}_{i=1}^t$, initial parameter $\btheta^{(0)}$.
    \STATE Define $\cL (\btheta) = \sum_{i=1}^t ( f(\xb_{i, a_i}; \btheta) - r_{i,a_i})^2/2 + m\lambda\|\btheta - \btheta^{(0)}\|_2^2/2$.
    \FOR{$j = 0, \dots, J-1$}
    \STATE $\btheta^{(j+1)} = \btheta^{(j)} - \eta\nabla \cL(\btheta^{(j)})$
    \ENDFOR 
    \STATE \textbf{Return} $\btheta^{(J)}$.
	\end{algorithmic}
\end{algorithm}

\paragraph{Comparison with Existing Algorithms}
We compare \algname\ with other neural contextual bandit algorithms. \citet{allesiardo2014neural} proposed NeuralBandit which consists of $K$ neural networks. It uses a committee of networks to compute the score of each action and chooses an action with the $\epsilon$-greedy strategy. In contrast, our \algname\ uses upper confidence bound-based exploration, which is more effective than $\epsilon$-greedy. In addition, our algorithm only uses one neural network instead of $K$ networks, thus can be computationally more efficient.

\citet{lipton2018bbq} used Thompson sampling on deep neural networks (through variational inference) in reinforcement learning; a variant is proposed by \citet{azizzadenesheli2018efficient} that works well on a set of Atari benchmarks. 
\citet{riquelme2018deep} proposed NeuralLinear, which uses the first $L-1$ layers of a $L$-layer DNN to learn a representation, then applies Thompson sampling on the last layer to choose action. \citet{zahavy2019deep} proposed a NeuralLinear with limited memory (NeuralLinearLM), which also uses the first $L-1$ layers of a $L$-layer DNN to learn a representation and applies Thompson sampling on the last layer. Instead of computing the exact mean and variance in Thompson sampling, NeuralLinearLM only computes their approximation. Unlike NeuralLinear and NeuralLinearLM, \algname\ uses the entire DNN to learn the representation and constructs the upper confidence bound based on the random feature mapping defined by the neural network gradient.  Finally, \citet{kveton20randomized} studied the use of reward perturbation for exploration in neural network-based bandit algorithms.

\noindent\textbf{A Variant of \algname} called $\algname_0$ is described in Appendix~\ref{sec:original}.  It can be viewed as a simplified version of $\algname$ where only the first-order Taylor approximation of the neural network around the initialized parameter is updated through online ridge regression. 
In this sense, $\algname_0$ can be seen as KernelUCB~\citep{valko2013finite} specialized to the Neural Tangent Kernel~\citep{jacot2018neural}, or LinUCB~\citep{li2010contextual} with Neural Tangent Random Features~\citep{cao2019generalization2}.

While this variant has a comparable regret bound as \algname, we expect the latter to be stronger in practice. Indeed, as shown by \citet{allen2019can}, the Neural Tangent Kernel does not seem to completely realize the representation power of neural networks in supervised learning.
A similar phenomenon will be demonstrated for contextual bandit learning in Section~\ref{sec:experiments}.

\section{Regret Analysis}\label{sec:regret_analysis}

This section analyzes the regret of \algname. 
Recall that $\{\xb^i\}_{i=1}^{TK}$ is the collection of all $\{\xb_{t,a}\}$.  Our regret analysis is built upon the recently proposed neural tangent kernel matrix \citep{jacot2018neural}:

\begin{definition}[\citet{jacot2018neural,cao2019generalization2}]\label{def:ntk}
Let $\{\xb^i\}_{i=1}^{TK}$ be a set of contexts.  Define
\begin{align*}
    \tilde \Hb_{i,j}^{(1)} &= \bSigma_{i,j}^{(1)} = \la \xb^i, \xb^j\ra, ~~~~~~~~  \Ab_{i,j}^{(l)} = 
    \begin{pmatrix}
    \bSigma_{i,i}^{(l)} & \bSigma_{i,j}^{(l)} \\
    \bSigma_{i,j}^{(l)} & \bSigma_{j,j}^{(l)} 
    \end{pmatrix},\notag \\
    \bSigma_{i,j}^{(l+1)} &= 2\EE_{(u, v)\sim N(\zero, \Ab_{i,j}^{(l)})} \left[\sigma(u)\sigma(v)\right],\notag \\
    \tilde \Hb_{i,j}^{(l+1)} &= 2\tilde \Hb_{i,j}^{(l)}\EE_{(u, v)\sim N(\zero, \Ab_{i,j}^{(l)})} \left[\sigma'(u)\sigma'(v)\right] + \bSigma_{i,j}^{(l+1)}.\notag
\end{align*}
Then, $\Hb = (\tilde \Hb^{(L)} + \bSigma^{(L)})/2$ is called the \emph{neural tangent kernel (NTK)} matrix on the context set. 
\end{definition}
In the above definition, the Gram matrix $\Hb$ of the NTK on the contexts $\{\xb^i\}_{i=1}^{TK}$ for $L$-layer neural networks is defined recursively from the input layer all the way to the output layer of the network.  Interested readers are referred to \citet{jacot2018neural} for more details about neural tangent kernels.

With Definition \ref{def:ntk}, we may state the following assumption on the contexts: $\{\xb^i\}_{i=1}^{TK}$. 
\begin{assumption}\label{assumption:input}
$\Hb \succeq \lambda_0\Ib$. 
Moreover, for any $1 \leq i \leq TK$, $\|\xb^i\|_2 = 1$ and $[\xb^i]_j =[\xb^i]_{j+d/2}$.
\end{assumption}
The first part of the assumption says that the neural tangent kernel matrix is non-singular, a mild assumption commonly made in the related literature \citep{du2018gradientdeep,arora2019exact,  cao2019generalization2}.  
It can be satisfied as long as \emph{no} two contexts in $\{\xb^i\}_{i=1}^{TK}$ are parallel. 
The second part is also mild and is just for convenience in analysis: for any context $\xb, \|\xb\|_2 = 1$, 
we can always construct a new context $\xb' = [\xb^\top, \xb^\top]^\top/\sqrt{2}$ to satisfy Assumption \ref{assumption:input}. It can be verified that if $\btheta_0$ is initialized as in \algname,
then $f(\xb^i; \btheta_0) = 0$ for any $i \in [TK]$.

Next we define the effective dimension of the neural tangent kernel matrix. 
\begin{definition}\label{def:effective_dimension}
The effective dimension $\tilde d$ of the neural tangent kernel matrix on contexts $\{\xb^i\}_{i=1}^{TK}$ is defined as
\begin{align}
    \tilde d = \frac{\log \det (\Ib + \Hb/\lambda)}{\log (1+TK/\lambda)}.
\end{align}
\end{definition}
\begin{remark}\label{remark:effect}
The notion of effective dimension was first introduced by \citet{valko2013finite} for analyzing kernel contextual bandits, which was defined by the eigenvalues of any kernel matrix restricted to the given contexts. We adapt a similar but different definition of \citet{yang2019reinforcement}, which was used for the analysis of kernel-based Q-learning. Suppose the dimension of the reproducing kernel Hilbert space induced by the given kernel is $\hat d$ and the feature mapping $\bpsi: \RR^d \rightarrow \RR^{\hat d}$ induced by the given kernel satisfies $\|\bpsi(\xb)\|_2 \leq 1$ for any $\xb\in \RR^d$. Then, it can be verified that $\tilde d \leq \hat d$, as shown in Appendix \ref{sec:val}. 
\CC{Intuitively, $\tilde d$ measures how quickly the eigenvalues of $\Hb$ diminish, and only depends on $T$ logarithmically in several special cases~\citep{valko2013finite}.}
\end{remark}

Now we are ready to present the main result, which provides the regret bound  $R_T$ of Algorithm~\ref{algorithm:2}. 
\begin{theorem}\label{thm:newalgorithm}
Let $\tilde d$ be the effective dimension, and $\hb = [h(\xb^i)]_{i=1}^{TK} \in \RR^{TK}$.  There exist constant $C_1, C_2>0$, such that for any $\delta \in(0,1)$, if
\begin{align}
&m  \geq \text{poly}(T, L, K, \lambda^{-1}, \lambda_0^{-1}, S^{-1}, \log(1/\delta)), \\
&\eta = C_1( mTL +  m \lambda)^{-1},\notag 
\end{align}
$\lambda \geq \max\{1, S^{-2}\}$, and $S \geq \sqrt{2\hb^\top\Hb^{-1}\hb}$,  then with probability at least $1-\delta$,
the regret of Algorithm \ref{algorithm:2} satisfies
\begin{align}
     R_T 
     & \leq  3\sqrt{T}\sqrt{\tilde d \log (1+TK/\lambda) + 2}\notag \\
     &\qquad \cdot \bigg[\nu\sqrt{\tilde d \log (1+TK/\lambda) + 2-2\log\delta}\notag \\
     &\qquad  +  (\lambda +C_2TL)(1-   \lambda/(TL))^{J/2} \sqrt{T/\lambda}\notag \\
     &\qquad + + 2\sqrt{\lambda}S\bigg]+ 1.\label{corollary:regret-1}
 \end{align}
\end{theorem}
\begin{remark}
It is worth noting that, simply applying results for linear bandits to our algorithm would lead to a linear dependence of $p$ or $\sqrt{p}$ in the regret.  Such a bound is vacuous since in our setting $p$ would be very large compared with the number of rounds $T$ and the input context dimension $d$. In contrast, our regret bound only depends on $\tilde d$, which can be much smaller than $p$. 
\end{remark}

\begin{remark}
\CC{Our regret bound \eqref{corollary:regret-1} has a term $(\lambda +C_2TL)(1-   \lambda/(TL))^{J/2} \sqrt{T/\lambda}$, which characterizes the optimization error of Algorithm \ref{algorithm:GD} after $J$ iterations. 
Setting
\begin{align}
    J = 2\log\frac{\lambda S}{\sqrt{T}(\lambda+C_2TL)}\frac{TL}{\lambda} = \tilde O(TL/\lambda),\label{eq:jchoose}
\end{align}
which is independent of $m$, we have $(\lambda +C_2TL)(1-   \lambda/(TL))^{J/2} \sqrt{T/\lambda} \leq \sqrt{\lambda}S$, so the optimization error is dominated by $\sqrt{\lambda}S$.  Hence, the order of the regret bound is not affected by the error of optimization.
} 
\end{remark}

\begin{remark}\label{remark:rkhsnorm}
With $\nu$ and $\lambda$ treated as constants, $S = \sqrt{2\hb^\top\Hb^{-1}\hb}$, and $J$ given in \eqref{eq:jchoose}, the regret bound  \eqref{corollary:regret-1} becomes $R_T = \tilde O\Big( \sqrt{\tilde d T}\sqrt{\max\{\tilde d, \hb^\top\Hb^{-1}\hb\}}\Big)$. Specifically, if $h$ belongs to the RKHS $\cH$ induced by the neural tangent kernel with bounded RKHS norm $\|h\|_{\cH}$, we have $\|h\|_{\cH} \geq  \sqrt{\hb^\top\Hb^{-1}\hb}$; \CC{see Appendix \ref{sec:RKHSnorm} for more details.}  Thus our regret bound can be further written as
\begin{equation}
R_T = \tilde O\Big( \sqrt{\tilde d T}\sqrt{\max\{\tilde d, \|h\|_{\cH}\}}\Big).
\end{equation}
\end{remark}

The high-probability result in Theorem \ref{thm:newalgorithm} can be used to obtain a bound on the expected regret.
\begin{corollary}\label{coro:expectation}
Under the same conditions in Theorem \ref{thm:newalgorithm}, there exists a positive constant $C$ such that
\begin{align}
    &\EE [R_T] \notag \\
     & \leq  2+3\sqrt{T}\sqrt{\tilde d \log (1+TK/\lambda) + 2}\notag \\
     &\qquad \cdot\bigg[\nu\sqrt{\tilde d \log (1+TK/\lambda) + 2+2\log T}\notag \\
     &\qquad + 2\sqrt{\lambda}S +  (\lambda + CTL)(1-  \lambda/(TL))^{J/2} \sqrt{T/\lambda}\bigg] .\notag
\end{align}
\end{corollary}

\section{Proof of Main Result}\label{section:newtheoryproof}

This section outlines the proof of Theorem \ref{thm:newalgorithm}, which has to deal with the following technical challenges:
\begin{itemize}[leftmargin = *]
    \item We do not make parametric assumptions on the reward function as some previous work~\citep{filippi2010parametric, chu2011contextual,abbasi2011improved}. 
    \item To avoid strong parametric assumptions, we use overparameterized neural networks, which implies $m$ (and thus $p$) is very large.  Therefore, we need to make sure the regret bound is independent of $m$. 
    \item Unlike the \emph{static} feature mapping used in kernel bandit algorithms~\citep{valko2013finite}, \algname\ uses a neural network $f(\xb; \btheta_t)$ and its gradient $\gb(\xb;\btheta_t)$ as a \emph{dynamic} feature mapping depending on $\btheta_t$. 
    This difference makes the analysis of \algname\ more difficult.
\end{itemize}

These challenges are addressed by the following technical lemmas, whose proofs are gathered in the appendix.
\begin{lemma}\label{lemma:equal}
There exists a positive constant $\bar C$ such that for any $\delta \in (0,1)$, if $m \geq \bar CT^4K^4L^6\log(T^2K^2L/\delta)/\lambda_0^4$, then with probability at least $1-\delta$, there exists a $\btheta^* \in \RR^p$ such that
\begin{align}
    &h(\xb^i) = \la \gb(\xb^i; \btheta_0), \btheta^* - \btheta_0\ra,\notag \\
    &\sqrt{m}\|\btheta^* - \btheta_0\|_2 \leq \sqrt{2\hb^\top\Hb^{-1}\hb}, \label{lemma:equal_0}
\end{align}
for all $i \in [TK]$. 
\end{lemma}
Lemma \ref{lemma:equal} suggests that with high probability, \CC{the reward function restricted to $\{\xb^i\}_{i=1}^{TK}$ can be regarded as a linear function of $\gb(\xb^i; \btheta_0)$ parameterized by $\btheta^* - \btheta_0$, where $\btheta^*$ lies in a ball centered at $\btheta_0$. Note that here $\btheta^*$ is not a ground truth parameter for the reward function. Instead, it is introduced only for the sake of analysis.}
Equipped with Lemma \ref{lemma:equal}, we can utilize existing results on linear bandits \citep{abbasi2011improved} to  
show that with high probability, $\btheta^* $ lies in the sequence of confidence sets.
\begin{lemma}\label{lemma:newinset}
There exist positive constants $\bar C_1$ and $\bar C_2$ such that for any $\delta \in (0,1)$, if $\eta \leq \bar C_1(TmL + m\lambda)^{-1}$ and 
\begin{align*}
m \geq \bar C_2\max\big\{ & T^7\lambda^{-7}L^{21}(\log m)^3, \\ & \lambda^{-1/2}L^{-3/2}(\log(TKL^2/\delta))^{3/2}\big\},
\end{align*}
then with probability at least $1-\delta$,
we have $\|\btheta_t - \btheta_0\|_2 \leq 2\sqrt{t/(m\lambda)}$ and $\|\btheta^* - \btheta_t\|_{ \Zb_t} \leq \gamma_t/\sqrt{m}$ for all $t \in [T]$, where $\gamma_t$ is defined in Algorithm \ref{algorithm:2}.
\end{lemma}

\begin{lemma}\label{lemma:newboundh}
Let $a_t^* = \argmax_{a \in [K]}h(\xb_{t, a})$. There exists a positive constant $\bar C$ such that for any $\delta \in (0,1)$, if $\eta$ and $m$ satisfy the same conditions as in Lemma~\ref{lemma:newinset},  
then with probability at least $1-\delta$, we have
\begin{align}
    &h\big(\xb_{t,a_t^*}\big) - h\big(\xb_{t,a_t}\big)\notag \\
     & \leq 2\gamma_{t-1}\min\bigg\{\|\gb(\xb_{t,a_t}; \btheta_{t-1})/\sqrt{m}\|_{\Zb_{t-1}^{-1}}, 1\bigg\} \notag \\
     & \qquad+ \bar C\big(Sm^{-1/6}\sqrt{\log m}T^{7/6}\lambda^{-1/6}L^{7/2}\notag \\
     &\qquad + m^{-1/6}\sqrt{\log m} T^{5/3}\lambda^{-2/3}L^3\big).\notag 
\end{align}
\end{lemma}
Lemma~\ref{lemma:newboundh} gives an upper bound for $h\big(\xb_{t,a_t^*}\big) - h\big(\xb_{t,a_t}\big)$, which can be used to bound the regret $R_T$. It is worth noting that $\gamma_t$ has a term $\log \det \Zb_t$. A trivial upper bound of $\log \det \Zb_t$ would result in a quadratic dependence on the network width $m$, since the dimension of $\Zb_t$ is $p = md+m^2(L-2)+m$.  Instead, we use the next lemma to establish an $m$-independent upper bound. 
The dependence on $\tilde{d}$ is similar to \citet[Lemma~4]{valko2013finite}, but the proof is different as our notion of effective dimension is different.

\begin{lemma}\label{lemma:newboundregret}
There exist positive constants $\{\bar C_i\}_{i=1}^3$ such that for any $\delta \in (0,1)$, if $m \geq \bar C_1\max\big\{T^7\lambda^{-7}L^{21}(\log m)^3, T^6K^6L^6(\log(TKL^2/\delta))^{3/2}\big\}$ and $\eta \leq \bar C_2(TmL + m\lambda)^{-1}$, 
then with probability at least $1-\delta$, we have
{\small
\begin{align*}
     &\sqrt{\sum_{t=1}^T\gamma_{t-1}^2\min\bigg\{\|\gb(\xb_{t,a_t}; \btheta_{t-1})/\sqrt{m}\|_{\Zb_{t-1}^{-1}}^2, 1\bigg\}} \notag \\
     & \leq  \sqrt{\tilde d\log(1+TK/\lambda) + \Gamma_1}\notag \\
     &\qquad \bigg[\Gamma_2 \bigg(\nu\sqrt{\tilde d\log(1+TK/\lambda)  + \Gamma_1-2\log\delta} + \sqrt{\lambda}S\bigg)\notag \\
    &\qquad +  (\lambda +\bar C_3tL)\Big[(1- \eta m \lambda)^{J/2} \sqrt{T/\lambda} + \Gamma_3(1+\sqrt{T/\lambda})\Big]\bigg]\notag,
\end{align*}
}
where 
\begin{align}
    &\Gamma_1 = 1+\bar C_3m^{-1/6}\sqrt{\log m} L^4T^{5/3}\lambda^{-1/6}, \notag\\
    &\Gamma_2 = \sqrt{1+\bar C_3m^{-1/6}\sqrt{\log m}L^4T^{7/6}\lambda^{-7/6}},\notag\\
    &\Gamma_3 = m^{-1/6}\sqrt{\log m}L^{7/2}T^{5/3}\lambda^{-5/3}.\notag
\end{align}
\end{lemma}
We are now ready to prove the main result.
\begin{proof}[Proof of Theorem \ref{thm:newalgorithm}]
Lemma~\ref{lemma:newboundh} implies that the total regret $R_T$ can be bounded as follows with a constant $C_1>0$:
\begin{align}
    R_T &= \sum_{t=1}^T \big[h\big(\xb_{t,a_t^*}\big) - h\big(\xb_{t,a_t}\big)\big]\notag \\
     & \leq 2\sum_{t=1}^T\gamma_{t-1}\min\bigg\{\|\gb(\xb_{t,a_t}; \btheta_{t-1})/\sqrt{m}\|_{\Zb_{t-1}^{-1}}, 1\bigg\} \notag \\
     & \qquad+ C_1\big(Sm^{-1/6}\sqrt{\log m}T^{13/6}\lambda^{-1/6}L^{7/2}\notag \\
     &\qquad + m^{-1/6}\sqrt{\log m} T^{8/3}\lambda^{-2/3}L^3\big).\notag 
\end{align}
It can be further bounded as follows:
 \begin{align}
      R_T& \leq 2\sqrt{T\sum_{t=1}^T\gamma_{t-1}^2\min\bigg\{\|\gb(\xb_{t,a_t}; \btheta_{t-1})/\sqrt{m}\|_{\Zb_{t-1}^{-1}}^2, 1\bigg\}} \notag \\
     & \qquad+ C_1\big(Sm^{-1/6}\sqrt{\log m}T^{13/6}\lambda^{-1/6}L^{7/2}\notag \\
     &\qquad + m^{-1/6}\sqrt{\log m} T^{8/3}\lambda^{-2/3}L^3\big)\notag\\
    & \leq 
    2\sqrt{T}\cdot  \sqrt{\tilde d\log(1+TK/\lambda) + \Gamma_1}\notag \\
     &\qquad \bigg[\Gamma_2 \bigg(\nu\sqrt{\tilde d\log(1+TK/\lambda)  + \Gamma_1-2\log\delta} + \sqrt{\lambda}S\bigg)\notag \\
    &\qquad + (\lambda + C_2TL)\Big[(1- \eta m \lambda)^{J/2} \sqrt{T/\lambda}\notag \\
    &\qquad + \Gamma_3(1+\sqrt{T/\lambda})\Big]\bigg]\notag\\
    & \qquad+ C_1\big(Sm^{-1/6}\sqrt{\log m}T^{13/6}\lambda^{-1/6}L^{7/2}\notag \\
     &\qquad + m^{-1/6}\sqrt{\log m} T^{8/3}\lambda^{-2/3}L^3\big)\notag\\
     & \leq 3\sqrt{T}\sqrt{\tilde d \log (1+TK/\lambda) + 2}\notag \\
     &\qquad \cdot \bigg[\nu\sqrt{\tilde d \log (1+TK/\lambda) + 2-2\log\delta}\notag \\
     &\qquad  + (\lambda+C_3TL)(1- \eta m \lambda)^{J/2} \sqrt{T/\lambda}\notag \\
     &\qquad  + 2\sqrt{\lambda}S\bigg]+ 1,\notag
 \end{align}
 where $C_1, C_2, C_3$ are positive constants, the first inequality is due to Cauchy-Schwarz inequality, the second inequality due to Lemma~\ref{lemma:newboundregret}, and the third inequality holds for sufficiently large $m$. This completes our proof.
\end{proof}

\begin{figure*}[t]
	\centering
		\subfigure[$h_1(\xb) = 10(\xb^\top\ab)^2$]{\includegraphics[width=0.32\linewidth]{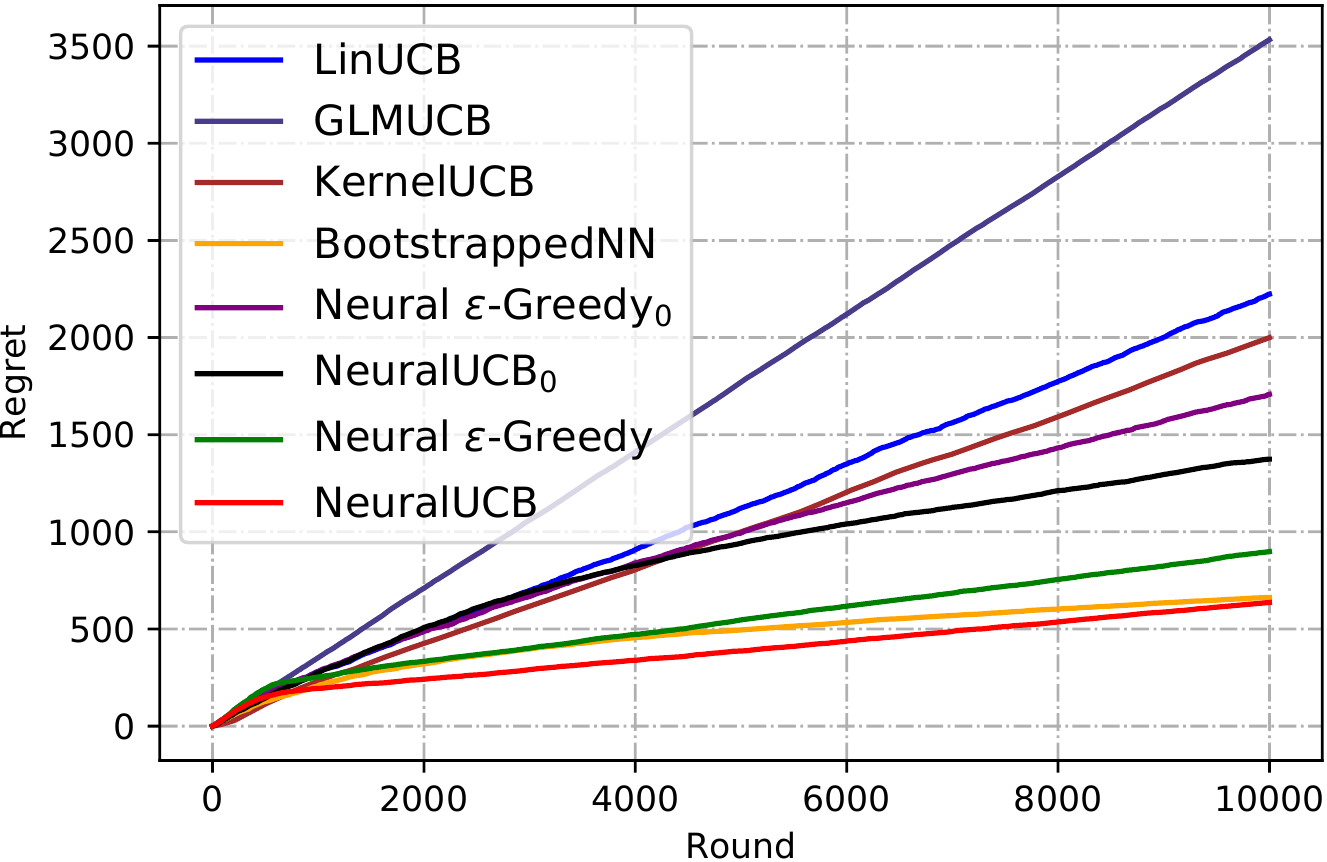}
    \label{fig:poly}}	
		\subfigure[$h_2(\xb) = \xb^\top\Ab^\top\Ab\xb$]{\includegraphics[width=0.32\linewidth]{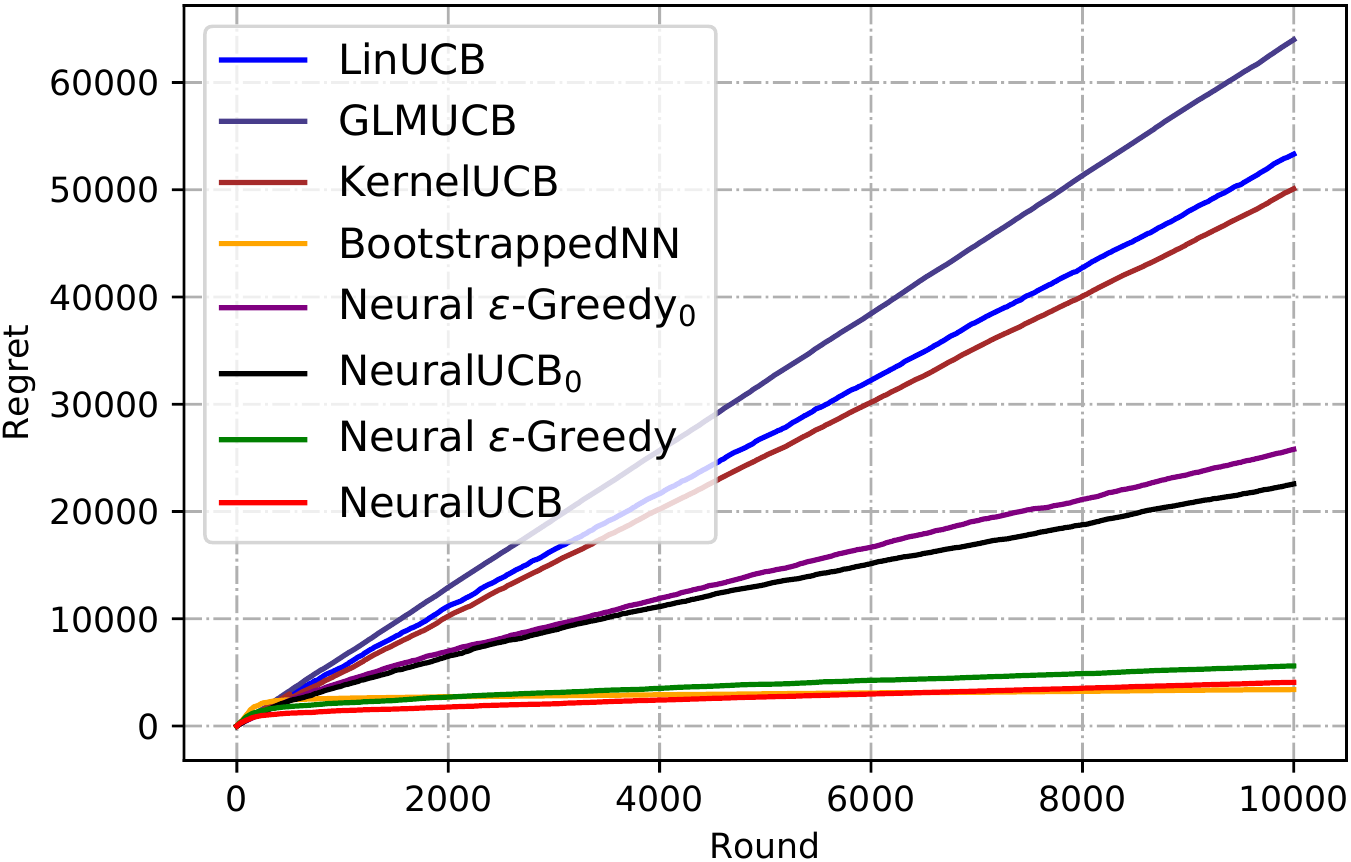}
		\label{fig:matrix}}
		\subfigure[$h_3(\xb) = \cos(3\xb^\top\ab)$]{\includegraphics[width=0.32\linewidth]{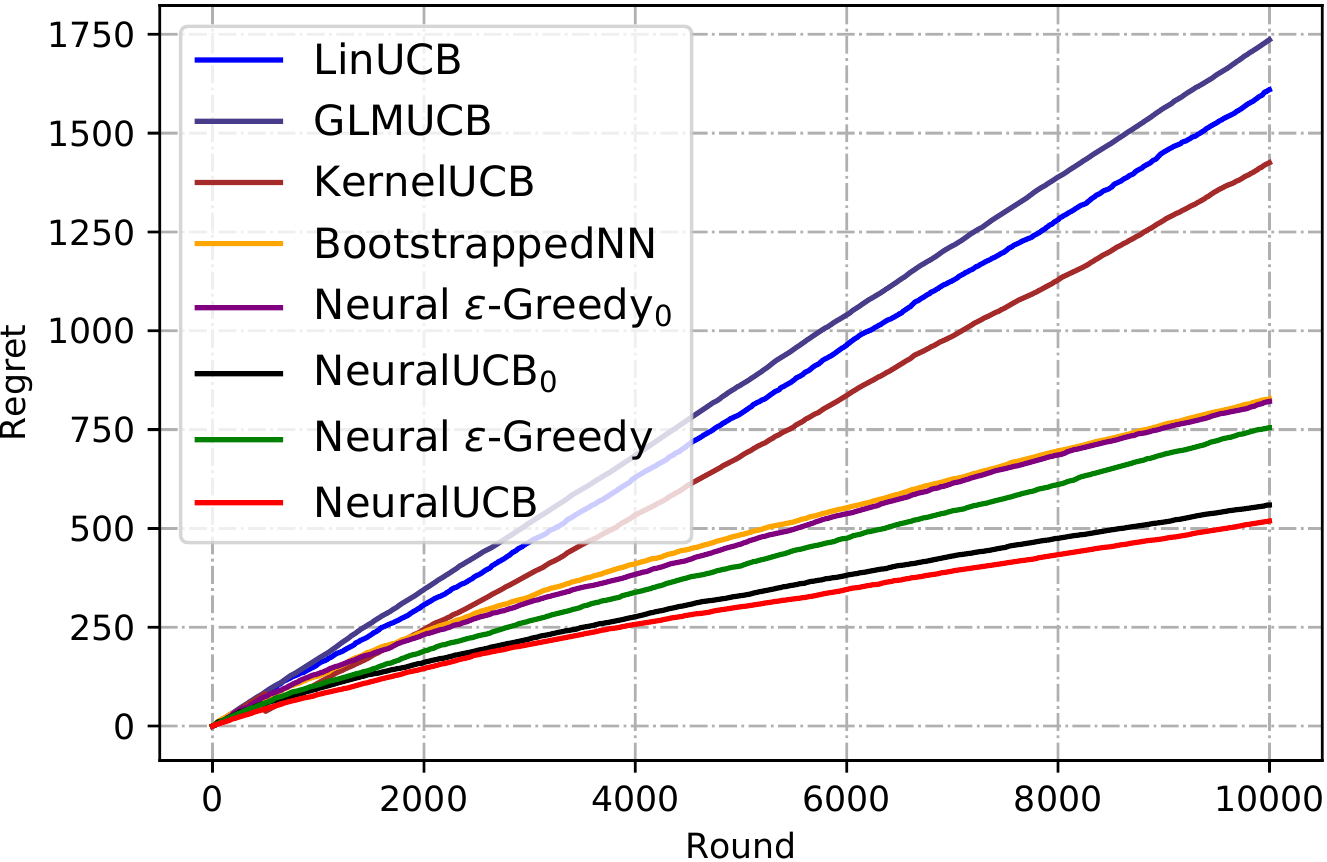}
		\label{fig:cos}}	
	\caption{Comparison of $\algname$ and baseline algorithms on synthetic datasets.}\label{fig1} 
\end{figure*}

\section{Related Work}

\paragraph{Contextual Bandits}
There is a line of extensive work on linear bandits~\citep[e.g.,][]{abe2003reinforcement,auer2002using,abe2003reinforcement,dani2008stochastic, rusmevichientong2010linearly,li2010contextual,chu2011contextual,abbasi2011improved}. 
Many of these algorithms are based on the idea of upper confidence bounds, and are shown to achieve near-optimal regret bounds.
Our algorithm is also based on UCB exploration, and the regret bound reduces to that of \citet{abbasi2011improved} in the linear case.

To deal with nonlinearity, a few authors have considered generalized linear bandits~\citep{filippi2010parametric, li2017provably,jun17scalable}, 
where the reward function is a composition of a linear function and a (nonlinear) link function.
Such models are special cases of what we study in this work.

More general nonlinear bandits without making strong modeling assumptions have also be considered.  One line of work is the family of expert learning algorithms~\citep{auer02nonstochastic,beygelzimer11contextual} that typically have a time complexity linear in the number of experts (which in many cases can be exponential in the number of parameters).

A second approach is to reduce a bandit problem to supervised learning, such as the epoch-greedy algorithm~\citep{langford08epoch} that has a non-optimal $O(T^{2/3})$ regret.  Later, \citet{agarwal14taming} develop an algorithm that enjoys a near-optimal regret, but relies on an oracle, whose implementation still requires proper modeling assumptions.

A third approach uses nonparametric modeling, such as \NEW{perceptrons \citep{kakade2008efficient}, random forests \citep{feraud2016random}, }Gaussian processes and kernels~\citep{kleinberg2008multi,srinivas2009gaussian,krause2011contextual,bubeck2011x}. 
The most relevant is by \citet{valko2013finite}, who assumed that the reward function lies in an RKHS with bounded RKHS norm and developed a UCB-based algorithm.
They also proved an $\tilde O(\sqrt{\tilde d T})$ regret, where $\tilde d$ is a form of effective dimension similar to ours.
Compared with these interesting works, our neural network-based algorithm avoids the need to carefully choose a good kernel or metric, and can be computationally more efficient in large-scale problems. \NEW{Recently, \citet{foster2020beyond} proposed contextual bandit algorithms with regression oracles which achieve a dimension-independent $O(T^{3/4})$ regret. Compared with \citet{foster2020beyond}, $\algname$ achieves a dimension-dependent $\tilde O(\tilde d\sqrt{T})$ regret with a better dependence on the time horizon.}

\paragraph{Neural Networks}
Substantial progress has been made to understand the expressive power of DNNs, in connection to the network depth~\citep{telgarsky2015representation,telgarsky2016benefits,liang2016deep,yarotsky2017error,yarotsky2018optimal,hanin2017universal}, as well as network width~\citep{lu2017expressive,hanin2017approximating}.  
The present paper on neural contextual bandit algorithms is inspired by these theoretical justifications and empirical evidence in the literature.

Our regret analysis for \algname\ makes use of recent advances in optimizing a DNN.  A series of works show that (stochastic) gradient descent can find global minima of the training loss \citep{li2018learning,du2018gradient,allen2018convergence,du2018gradientdeep, zou2018stochastic,zou2019improved}. 
For the generalization of DNNs, a number of authors~\citep{daniely2017sgd,cao2019generalization2,cao2019generalization1,arora2019exact, chen2019much} show that by using (stochastic) gradient descent, the parameters of a DNN are located in a particular regime and the generalization bound of DNNs can be characterized by the best function in the corresponding neural tangent kernel space \citep{jacot2018neural}.

\section{Experiments}
\label{sec:experiments}

In this section, we evaluate \algname\ empirically and compare it with seven representative baselines: (1) LinUCB, which is also based on UCB but adopts a linear representation; (2) GLMUCB \citep{filippi2010parametric}, which applies a nonlinear link function over a linear function; (3) KernelUCB \citep{valko2013finite}, a kernelised UCB algorithm which makes use of a predefined kernel function;
(4) BootstrappedNN \citep{efron1982jackknife, riquelme2018deep}, which simultaneously trains a set of neural networks using bootstrapped samples and at every round chooses an action based on the prediction of a randomly picked model; 
(5) Neural $\epsilon$-Greedy, which replaces the UCB-based exploration in Algorithm \ref{algorithm:2} by $\epsilon$-greedy; \CC{(6) $\algname_0$, as described in Section~\ref{sec:algorithm}; and (7) Neural $\epsilon$-Greedy$_0$, same as $\algname_0$ but with $\epsilon$-greedy exploration}.
We use the cumulative regret as the performance metric.

\begin{figure*}[t]
	\centering
		\subfigure[\dsfont{covertype}]{\includegraphics[width=0.32\linewidth]{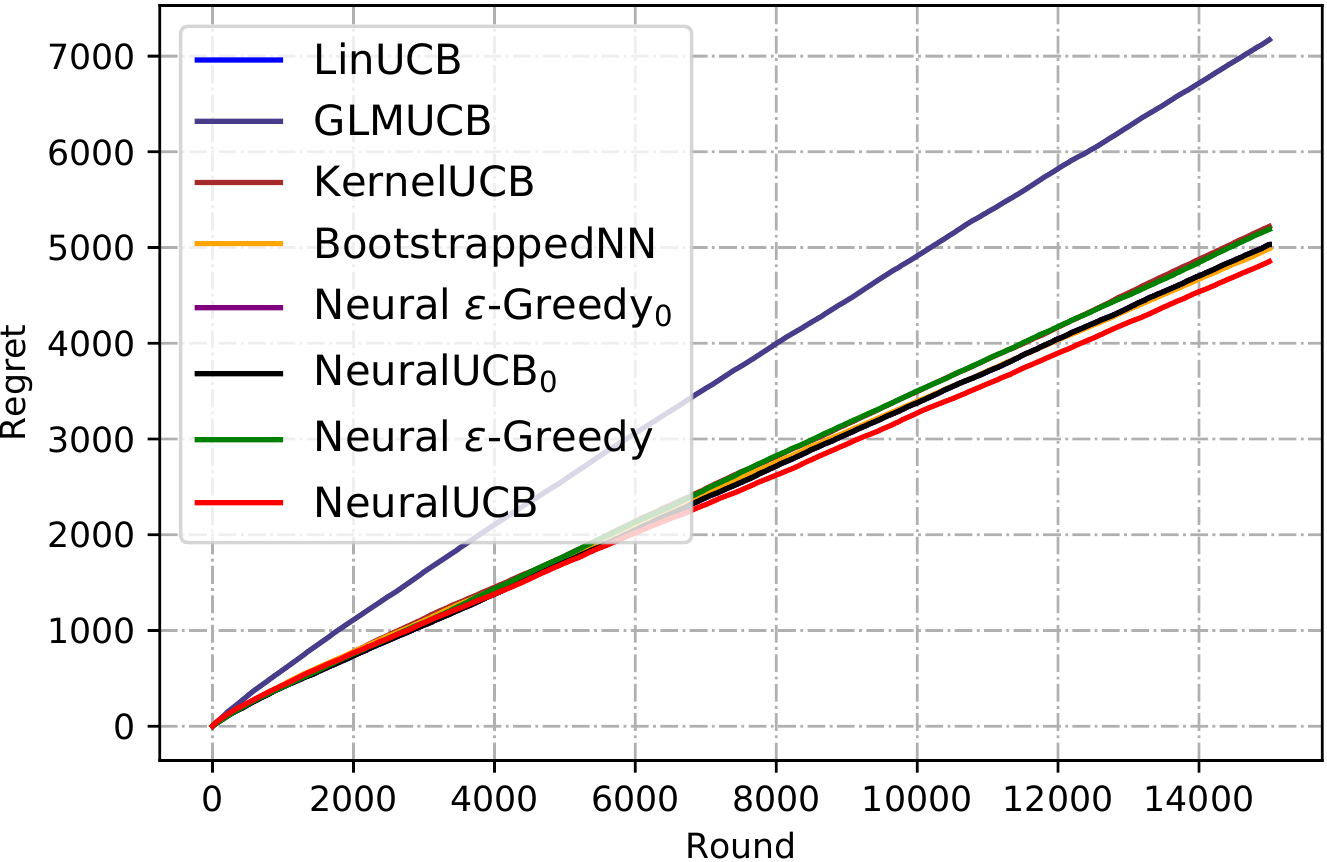}
		\label{fig:covtype}}
		\subfigure[\dsfont{magic}]{\includegraphics[width=0.32\linewidth]{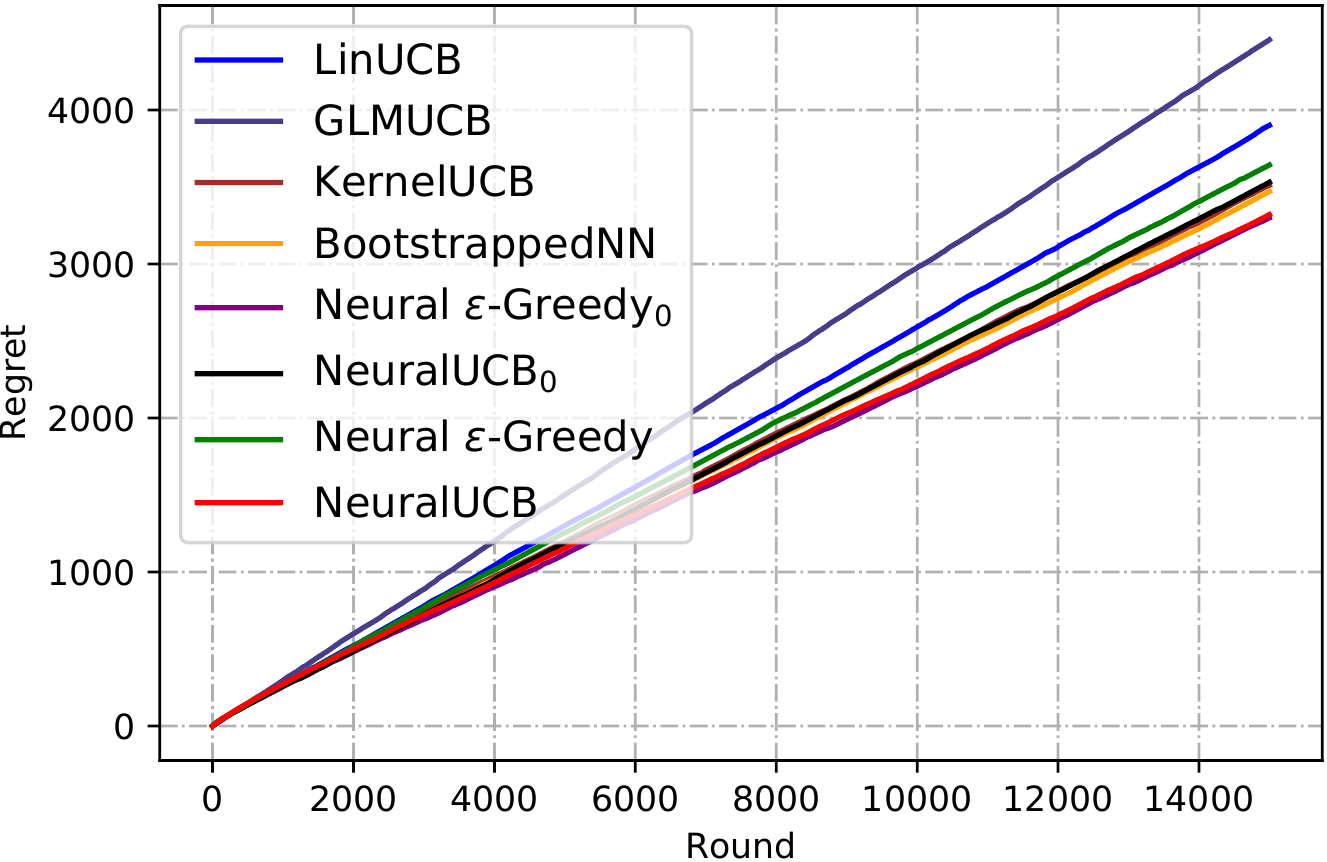}
		\label{fig:magic}}
		
			\subfigure[\dsfont{statlog}]{\includegraphics[width=0.32\linewidth]{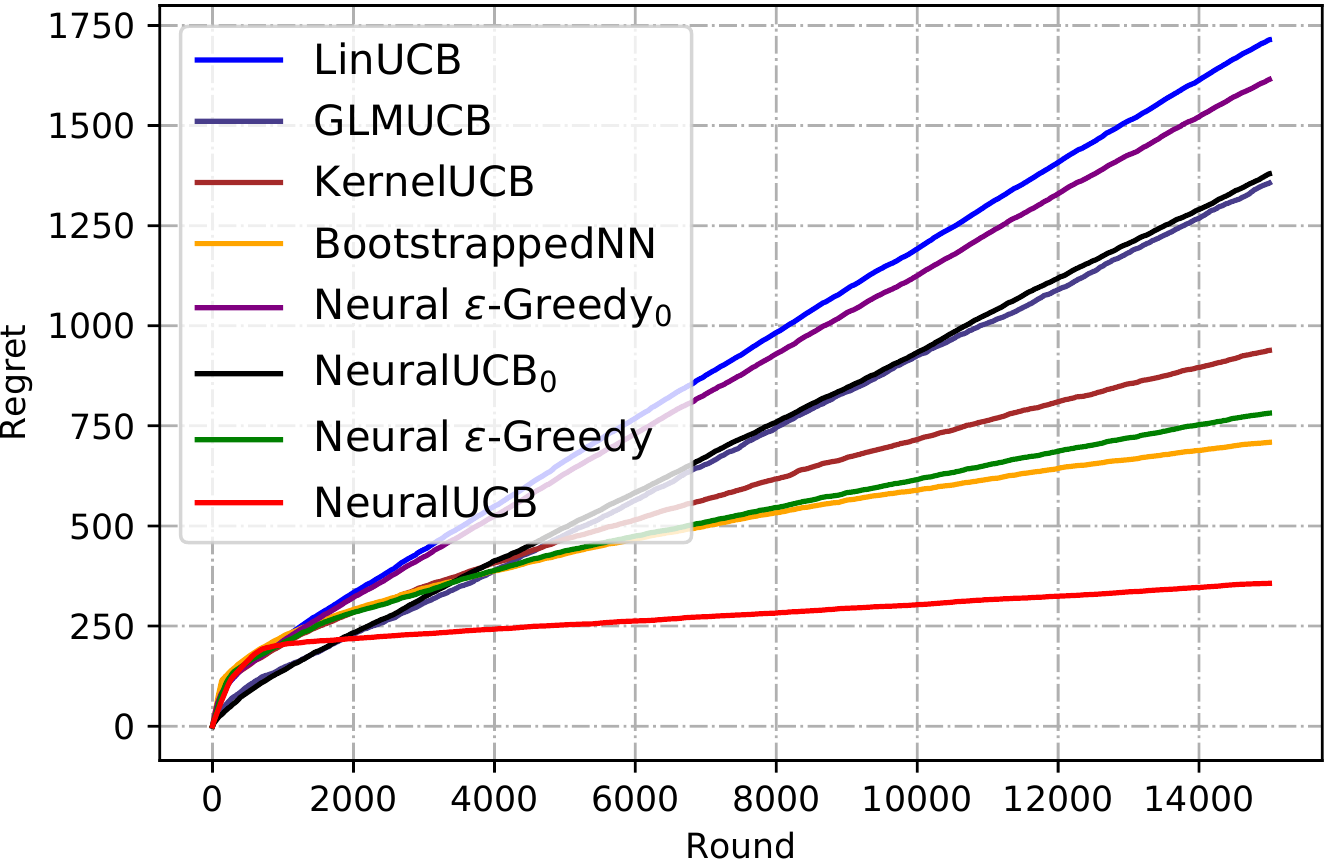}
		\label{fig:statlog}}
		\subfigure[\dsfont{mnist}]{\includegraphics[width=0.32\linewidth]{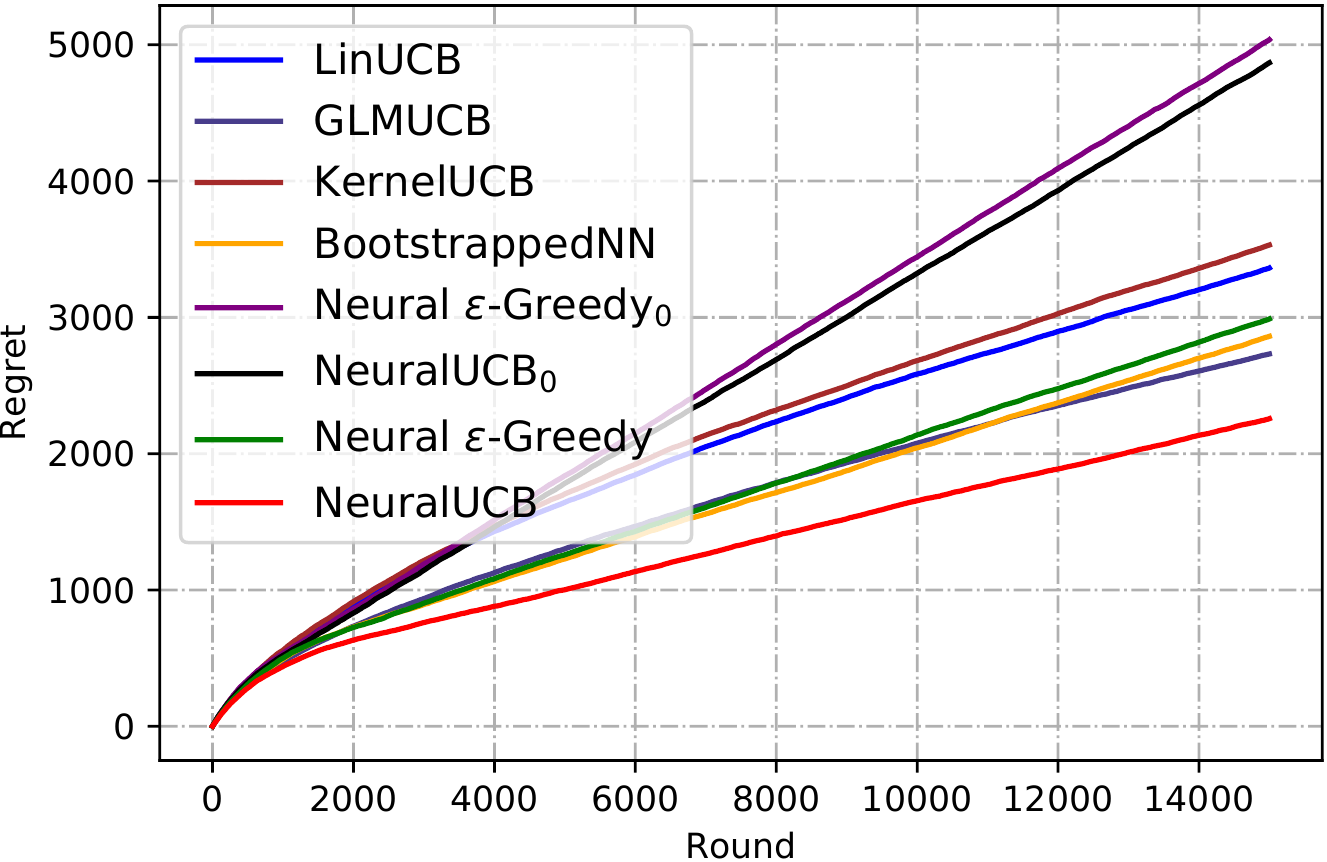}
		\label{fig:mnist}}
	\caption{Comparison of $\algname$ and baseline algorithms on real-world datasets.}\label{fig2} 
\end{figure*}

\subsection{Synthetic Datasets}\label{sec:simulated}

In the first set of experiments, we use contextual bandits with context dimension $d = 20$ and $K=4$ actions.  The number of rounds $T = 10~000$. The contextual vectors $\{\xb_{1,1},\dots,\xb_{T,K}\}$ are chosen uniformly at random from the unit ball.  The reward function $h$ is one of the following:
\begin{align*}
h_1(\xb) &= 10(\xb^\top\ab)^2, \\
h_2(\xb) &= \xb^\top\Ab^\top\Ab\xb, \\ h_3(\xb) &= \cos(3\xb^\top\ab)\,,
\end{align*}
where each entry of $\Ab \in \RR^{d \times d}$ is randomly generated from $N(0,1)$, $\ab$ is randomly generated from uniform distribution over unit ball.
For each $h_i(\cdot)$, 
the reward 
is generated by $r_{t,a} = h_i(\xb_{t,a}) + \xi_t$, where $\xi_t \sim N(0,1)$.

Following \citet{li2010contextual}, we implement LinUCB using a constant $\alpha$ (for the variance term in the UCB). We do a grid search for $\alpha$ over $\{0.01, 0.1, 1, 10\}$. For GLMUCB, we use the sigmoid function as the link function and adapt the online Newton step method to accelerate the computation \citep{zhang2016online, jun17scalable}. We do grid searches over $\{0.1, 1, 10\}$ for regularization parameter, $\{1, 10, 100\}$ for step size, $\{0.01, 0.1, 1\}$ for exploration parameter. For KernelUCB, we use the radial basis function (RBF) kernel with parameter $\sigma$, and set the regularization parameter to $1$.  Grid searches over $\{0.1,1, 10\}$ for $\sigma$ and $\{0.01, 0.1, 1, 10\}$ for the exploration parameter are done.
To accelerate the calculation, we stop adding contexts to KernelUCB after $1000$ rounds, following the same setting for Gaussian Process in \citet{riquelme2018deep}. 
For all five neural algorithms,
we choose a two-layer neural network $f(\xb; \btheta)= \sqrt{m}\Wb_2\sigma(\Wb_1\xb)$ with network width $m = 20$, where $\btheta = [\text{vec}(\Wb_1)^\top,\text{vec}(\Wb_2)^\top] \in \RR^{p}$ and $p = md+m=420$.\footnote{Note that the bound on the required network width $m$ is likely not tight.  Therefore, in experiments we choose $m$ to be relatively large, but not as large as theory suggests.}
Moreover, we set $\gamma_t=\gamma$ in \algname,
and do a grid search over $\{0.01,0.1,1,10\}$. For $\algname_0$, we do grid searches for $\nu$ over $\{0.1, 1, 10\}$, for $\lambda$ over $\{0.1, 1, 10\}$, for $\delta$ over $\{0.01, 0.1, 1\}$, for $S$ over 
$\{0.01,0.1,1,10\}$.
For Neural $\epsilon$-Greedy and Neural $\epsilon$-Greedy$_0$, we do a grid search for $\epsilon$ over $\{0.001,0.01,0.1,0.2\}$. For BootstrappedNN, we follow \citet{riquelme2018deep} to set the number of models to be $10$ and the transition probability to be $0.8$. 
To accelerate the training process, for BootstrappedNN, $\algname$ and Neural $\epsilon$-Greedy, we update the parameter $\btheta_t$ by \text{TrainNN} every $50$ rounds. We use stochastic gradient descent with batch size $50$, $J = t$ at round $t$, and do a grid search for step size $\eta$ over $\{0.001, 0.01, 0.1\}$.
For all grid-searched parameters, we choose the best of them for the comparison.
All experiments are repeated $10$ times, and the averaged results reported for comparison.

\subsection{Real-world Datasets}

\begin{table}[t]
\caption{Dataset statistics}
\label{tab:dataset}
\vskip 0.15in
\begin{center}
\begin{small}
\begin{sc}
\begin{tabular}{lccccr}
\toprule
     \multirow{2}{*}{Dataset}  &  \small{Cover-} & \multirow{2}{*}{magic} & \multirow{2}{*}{statlog} &\multirow{2}{*}{mnist}\\
    & \small{type} & & &\\
\midrule
    \small{feature}  &  \multirow{2}{*}{54} & \multirow{2}{*}{10} & \multirow{2}{*}{8} &\multirow{2}{*}{784}\\
    \small{dimension} & & & &\\
         \hline
    \small{number of}  & \multirow{2}{*}{7} & \multirow{2}{*}{2} & \multirow{2}{*}{7} & \multirow{2}{*}{10}\\
    \small{classes} & & & &\\
    \hline
    \small{number of} & \multirow{2}{*}{581012} & \multirow{2}{*}{19020} & \multirow{2}{*}{58000} & \multirow{2}{*}{60000}\\
    \small{instances} & & & &\\
\bottomrule
\end{tabular}
\end{sc}
\end{small}
\end{center}
\vskip -0.1in
\end{table}

We evaluate our algorithms on 
real-world datasets from the UCI Machine Learning Repository~\citep{Dua:2019}: \dsfont{covertype}, \dsfont{magic}, 
and \dsfont{statlog}. We also evaluate our algorithms on \dsfont{mnist} dataset \citep{lecun1998gradient}. 
These are all $K$-class classification datasets (Table~\ref{tab:dataset}), and are converted into $K$-armed contextual bandits~\citep{beygelzimer09offset}.
The number of rounds is set as $T = 15000$.
Following \citet{riquelme2018deep}, we create contextual bandit problems based on the prediction accuracy. 
In detail, to transform a classification problem with $k$-classes into a bandit problem, we adapts the disjoint model \citep{li2010contextual} which transforms each contextual vector $\xb \in \RR^d$ into $k$ vectors $\xb^{(1)} = (\xb, \zero,\dots, \zero), \dots, \xb^{(k)} = (\zero,\dots, \zero, \xb) \in \RR^{dk}$. The agent received regret $0$ if he classifies the context correctly, and $1$ otherwise. For all the algorithms, 
We reshuffle the order of contexts and repeat the experiment for $10$ runs. Averaged results are reported for comparison.

For LinUCB, GLMUCB and KernelUCB, we tune their parameters as Section \ref{sec:simulated} suggests. 
For BootstrappedNN, $\algname$, $\algname_0$, Neural $\epsilon$-Greedy and Neural $\epsilon$-Greedy$_0$, we choose a two-layer neural network 
with width $m = 100$.  For $\algname$ and $\algname_0$, since it is computationally expensive to store and compute a whole matrix $\Zb_t$, we use a diagonal matrix which consists of the diagonal elements of $\Zb_t$ to approximate $\Zb_t$. To accelerate the training process, for BootstrappedNN, $\algname$ and Neural $\epsilon$-Greedy, we update the parameter $\btheta_t$ by TrainNN every $100$ rounds starting from round 2000. We do grid searches for $\lambda$ over $\{10^{-i}\}, i = 1,2,3,4$, for $\eta$ over $\{2\times 10^{-i}, 5\times 10^{-i}\}, i=1,2,3,4$. We set $J = 1000$ and use stochastic gradient descent with batch size $500$ to train the networks. For the rest of parameters, we tune them as those in Section \ref{sec:simulated} and choose the best of them for comparison.

\subsection{Results}
\vskip -0.07in
Figures~\ref{fig1} and \ref{fig2} show the cumulative regret of all algorithms.
First, due to the nonlinearity of reward functions $h$, LinUCB fails to learn them for nearly all tasks. \NEW{GLMUCB is only able to learn the true reward functions for certain tasks due to its simple link function. }
In contrast, thanks to the neural network representation and efficient exploration, $\algname$ achieves 
a substantially lower regret. 
The performance of Neural $\epsilon$-Greedy is between the two. This suggests that while Neural $\epsilon$-Greedy can capture the nonlinearity of the underlying reward function, $\epsilon$-Greedy based exploration is not as effective as UCB based exploration. 
This confirms the effectiveness of \algname\ for contextual bandit problems with 
nonlinear reward functions. 
Second, it is worth noting that $\algname$ and Neural $\epsilon$-Greedy outperform $\algname_0$ and Neural $\epsilon$-Greedy$_0$.  This suggests that using deep neural networks to predict the reward function is better than using a fixed feature mapping associated with the Neural Tangent Kernel, which mirrors similar findings in supervised learning~\citep{allen2019can}.  Furthermore, we can see that KernelUCB is not as good as \algname, which suggests the limitation of simple kernels like RBF compared to flexible neural networks.
What's more, 
BootstrappedNN can be competitive, approaching the performance of \algname\ in some datasets. However, it requires to maintain and train multiple neural networks, so is computationally more expensive than our approach, especially in large-scale problems.

\section{Conclusion}
\vskip -0.07in
In this paper, we proposed \algname, a new algorithm for stochastic contextual bandits based on neural networks and upper confidence bounds. Building on recent advances in optimization and generalization of deep neural networks, we showed that for an arbitrary bounded reward function, our algorithm achieves an $\tilde O(\tilde d \sqrt{T})$ regret bound.  Promising empirical results on both synthetic and real-world data corroborated our theoretical findings, and suggested the potential of the algorithm in practice.

We conclude the paper with a suggested direction for future research.  Given the focus on UCB exploration in this work, a natural open question is provably efficient exploration based on randomized strategies, when DNNs are used.  These methods are effective in practice, but existing regret analyses are mostly for shallow (i.e., linear or generalized linear) models~\citep{chapelle2011empirical,agrawal2013thompson,russo18tutorial,kveton20randomized}.
Extending them to DNNs will be interesting. \NEW{Meanwhile, our current analysis of NeuralUCB is based on the NTK theory. While NTK facilitates the analysis, it has its own limitations, and we will leave the analysis of NeuralUCB beyond NTK as future work.} 

\section*{Acknowledgement}
We would like to thank the anonymous reviewers for their helpful comments. This research was sponsored in part by the National Science Foundation IIS-1904183 and IIS-1906169. The views and conclusions contained in this paper are those of the authors and should not be interpreted as representing any funding agencies.

\bibliography{reference}

\begin{thebibliography}{61}
\providecommand{\natexlab}[1]{#1}
\providecommand{\url}[1]{\texttt{#1}}
\expandafter\ifx\csname urlstyle\endcsname\relax
  \providecommand{\doi}[1]{doi: #1}\else
  \providecommand{\doi}{doi: \begingroup \urlstyle{rm}\Url}\fi

\bibitem[Abbasi-Yadkori et~al.(2011)Abbasi-Yadkori, P{\'a}l, and
  Szepesv{\'a}ri]{abbasi2011improved}
Abbasi-Yadkori, Y., P{\'a}l, D., and Szepesv{\'a}ri, C.
\newblock Improved algorithms for linear stochastic bandits.
\newblock In \emph{Advances in Neural Information Processing Systems}, pp.\
  2312--2320, 2011.

\bibitem[Abe et~al.(2003)Abe, Biermann, and Long]{abe2003reinforcement}
Abe, N., Biermann, A.~W., and Long, P.~M.
\newblock Reinforcement learning with immediate rewards and linear hypotheses.
\newblock \emph{Algorithmica}, 37\penalty0 (4):\penalty0 263--293, 2003.

\bibitem[Agarwal et~al.(2014)Agarwal, Hsu, Kale, Langford, Li, and
  Schapire]{agarwal14taming}
Agarwal, A., Hsu, D., Kale, S., Langford, J., Li, L., and Schapire, R.~E.
\newblock Taming the monster: A fast and simple algorithm for contextual
  bandits.
\newblock In \emph{Proceedings of the 31st International Conference on Machine
  Learning (ICML)}, pp.\  1638--1646, 2014.

\bibitem[Agrawal \& Goyal(2013)Agrawal and Goyal]{agrawal2013thompson}
Agrawal, S. and Goyal, N.
\newblock Thompson sampling for contextual bandits with linear payoffs.
\newblock In \emph{International Conference on Machine Learning}, pp.\
  127--135, 2013.

\bibitem[Allen-Zhu \& Li(2019)Allen-Zhu and Li]{allen2019can}
Allen-Zhu, Z. and Li, Y.
\newblock What can {ResNet} learn efficiently, going beyond kernels?
\newblock In \emph{Advances in Neural Information Processing Systems}, 2019.

\bibitem[Allen-Zhu et~al.(2019)Allen-Zhu, Li, and Song]{allen2018convergence}
Allen-Zhu, Z., Li, Y., and Song, Z.
\newblock A convergence theory for deep learning via over-parameterization.
\newblock In \emph{International Conference on Machine Learning}, pp.\
  242--252, 2019.

\bibitem[Allesiardo et~al.(2014)Allesiardo, F{\'e}raud, and
  Bouneffouf]{allesiardo2014neural}
Allesiardo, R., F{\'e}raud, R., and Bouneffouf, D.
\newblock A neural networks committee for the contextual bandit problem.
\newblock In \emph{International Conference on Neural Information Processing},
  pp.\  374--381. Springer, 2014.

\bibitem[Arora et~al.(2019)Arora, Du, Hu, Li, Salakhutdinov, and
  Wang]{arora2019exact}
Arora, S., Du, S.~S., Hu, W., Li, Z., Salakhutdinov, R., and Wang, R.
\newblock On exact computation with an infinitely wide neural net.
\newblock In \emph{Advances in Neural Information Processing Systems}, 2019.

\bibitem[Auer(2002)]{auer2002using}
Auer, P.
\newblock Using confidence bounds for exploitation-exploration trade-offs.
\newblock \emph{Journal of Machine Learning Research}, 3\penalty0
  (Nov):\penalty0 397--422, 2002.

\bibitem[Auer et~al.(2002)Auer, {Cesa-Bianchi}, Freund, and
  Schapire]{auer02nonstochastic}
Auer, P., {Cesa-Bianchi}, N., Freund, Y., and Schapire, R.~E.
\newblock The nonstochastic multiarmed bandit problem.
\newblock \emph{SIAM Journal on Computing}, 32\penalty0 (1):\penalty0 48--77,
  2002.

\bibitem[Azizzadenesheli et~al.(2018)Azizzadenesheli, Brunskill, and
  Anandkumar]{azizzadenesheli2018efficient}
Azizzadenesheli, K., Brunskill, E., and Anandkumar, A.
\newblock Efficient exploration through {Bayesian} deep {Q}-networks.
\newblock In \emph{2018 Information Theory and Applications Workshop (ITA)},
  pp.\  1--9. IEEE, 2018.

\bibitem[Beygelzimer \& Langford(2009)Beygelzimer and
  Langford]{beygelzimer09offset}
Beygelzimer, A. and Langford, J.
\newblock The offset tree for learning with partial labels.
\newblock In \emph{Proceedings of the 15th ACM SIGKDD International Conference
  on Knowledge Discovery and Data Mining}, pp.\  129--138, 2009.

\bibitem[Beygelzimer et~al.(2011)Beygelzimer, Langford, Li, Reyzin, and
  Schapire]{beygelzimer11contextual}
Beygelzimer, A., Langford, J., Li, L., Reyzin, L., and Schapire, R.~E.
\newblock Contextual bandit algorithms with supervised learning guarantees.
\newblock In \emph{Proceedings of the Fourteenth International Conference on
  Artificial Intelligence and Statistics}, pp.\  19--26, 2011.

\bibitem[Bubeck \& Cesa-Bianchi(2012)Bubeck and Cesa-Bianchi]{bubeck12regret}
Bubeck, S. and Cesa-Bianchi, N.
\newblock Regret analysis of stochastic and nonstochastic multi-armed bandit
  problems.
\newblock \emph{Foundations and Trends in Machine Learning}, 5\penalty0
  (1):\penalty0 1--122, 2012.

\bibitem[Bubeck et~al.(2011)Bubeck, Munos, Stoltz, and
  Szepesv{\'a}ri]{bubeck2011x}
Bubeck, S., Munos, R., Stoltz, G., and Szepesv{\'a}ri, C.
\newblock X-armed bandits.
\newblock \emph{Journal of Machine Learning Research}, 12\penalty0
  (May):\penalty0 1655--1695, 2011.

\bibitem[Cao \& Gu(2019)Cao and Gu]{cao2019generalization2}
Cao, Y. and Gu, Q.
\newblock Generalization bounds of stochastic gradient descent for wide and
  deep neural networks.
\newblock In \emph{Advances in Neural Information Processing Systems}, 2019.

\bibitem[Cao \& Gu(2020)Cao and Gu]{cao2019generalization1}
Cao, Y. and Gu, Q.
\newblock Generalization error bounds of gradient descent for learning
  over-parameterized deep relu networks.
\newblock In \emph{the Thirty-Fourth AAAI Conference on Artificial
  Intelligence}, 2020.

\bibitem[Chapelle \& Li(2011)Chapelle and Li]{chapelle2011empirical}
Chapelle, O. and Li, L.
\newblock An empirical evaluation of thompson sampling.
\newblock In \emph{Advances in neural information processing systems}, pp.\
  2249--2257, 2011.

\bibitem[Chen et~al.(2019)Chen, Cao, Zou, and Gu]{chen2019much}
Chen, Z., Cao, Y., Zou, D., and Gu, Q.
\newblock How much over-parameterization is sufficient to learn deep relu
  networks?
\newblock \emph{arXiv preprint arXiv:1911.12360}, 2019.

\bibitem[Chu et~al.(2011)Chu, Li, Reyzin, and Schapire]{chu2011contextual}
Chu, W., Li, L., Reyzin, L., and Schapire, R.
\newblock Contextual bandits with linear payoff functions.
\newblock In \emph{Proceedings of the Fourteenth International Conference on
  Artificial Intelligence and Statistics}, pp.\  208--214, 2011.

\bibitem[Dani et~al.(2008)Dani, Hayes, and Kakade]{dani2008stochastic}
Dani, V., Hayes, T.~P., and Kakade, S.~M.
\newblock Stochastic linear optimization under bandit feedback.
\newblock 2008.

\bibitem[Daniely(2017)]{daniely2017sgd}
Daniely, A.
\newblock {SGD} learns the conjugate kernel class of the network.
\newblock In \emph{Advances in Neural Information Processing Systems}, pp.\
  2422--2430, 2017.

\bibitem[Du et~al.(2019{\natexlab{a}})Du, Lee, Li, Wang, and
  Zhai]{du2018gradientdeep}
Du, S., Lee, J., Li, H., Wang, L., and Zhai, X.
\newblock Gradient descent finds global minima of deep neural networks.
\newblock In \emph{International Conference on Machine Learning}, pp.\
  1675--1685, 2019{\natexlab{a}}.

\bibitem[Du et~al.(2019{\natexlab{b}})Du, Zhai, Poczos, and
  Singh]{du2018gradient}
Du, S.~S., Zhai, X., Poczos, B., and Singh, A.
\newblock Gradient descent provably optimizes over-parameterized neural
  networks.
\newblock In \emph{International Conference on Learning Representations},
  2019{\natexlab{b}}.
\newblock URL \url{https://openreview.net/forum?id=S1eK3i09YQ}.

\bibitem[Dua \& Graff(2017)Dua and Graff]{Dua:2019}
Dua, D. and Graff, C.
\newblock {UCI} machine learning repository, 2017.
\newblock URL \url{http://archive.ics.uci.edu/ml}.

\bibitem[Efron(1982)]{efron1982jackknife}
Efron, B.
\newblock \emph{The jackknife, the bootstrap, and other resampling plans},
  volume~38.
\newblock Siam, 1982.

\bibitem[F{\'e}raud et~al.(2016)F{\'e}raud, Allesiardo, Urvoy, and
  Cl{\'e}rot]{feraud2016random}
F{\'e}raud, R., Allesiardo, R., Urvoy, T., and Cl{\'e}rot, F.
\newblock Random forest for the contextual bandit problem.
\newblock In \emph{Artificial Intelligence and Statistics}, pp.\  93--101,
  2016.

\bibitem[Filippi et~al.(2010)Filippi, Cappe, Garivier, and
  Szepesv{\'a}ri]{filippi2010parametric}
Filippi, S., Cappe, O., Garivier, A., and Szepesv{\'a}ri, C.
\newblock Parametric bandits: The generalized linear case.
\newblock In \emph{Advances in Neural Information Processing Systems}, pp.\
  586--594, 2010.

\bibitem[Foster \& Rakhlin(2020)Foster and Rakhlin]{foster2020beyond}
Foster, D.~J. and Rakhlin, A.
\newblock Beyond ucb: Optimal and efficient contextual bandits with regression
  oracles.
\newblock \emph{arXiv preprint arXiv:2002.04926}, 2020.

\bibitem[Goodfellow et~al.(2016)Goodfellow, Bengio, and
  Courville]{goodfellow2016deep}
Goodfellow, I., Bengio, Y., and Courville, A.
\newblock \emph{Deep Learning}.
\newblock MIT Press, 2016.
\newblock \url{http://www.deeplearningbook.org}.

\bibitem[Hanin(2017)]{hanin2017universal}
Hanin, B.
\newblock Universal function approximation by deep neural nets with bounded
  width and {ReLU} activations.
\newblock \emph{arXiv preprint arXiv:1708.02691}, 2017.

\bibitem[Hanin \& Sellke(2017)Hanin and Sellke]{hanin2017approximating}
Hanin, B. and Sellke, M.
\newblock Approximating continuous functions by {ReLU} nets of minimal width.
\newblock \emph{arXiv preprint arXiv:1710.11278}, 2017.

\bibitem[Jacot et~al.(2018)Jacot, Gabriel, and Hongler]{jacot2018neural}
Jacot, A., Gabriel, F., and Hongler, C.
\newblock Neural tangent kernel: Convergence and generalization in neural
  networks.
\newblock In \emph{Advances in neural information processing systems}, pp.\
  8571--8580, 2018.

\bibitem[Jun et~al.(2017)Jun, Bhargava, Nowak, and Willett]{jun17scalable}
Jun, K.-S., Bhargava, A., Nowak, R.~D., and Willett, R.
\newblock Scalable generalized linear bandits: Online computation and hashing.
\newblock In \emph{Advances in Neural Information Processing Systems 30
  (NIPS)}, pp.\  99--109, 2017.

\bibitem[Kakade et~al.(2008)Kakade, Shalev-Shwartz, and
  Tewari]{kakade2008efficient}
Kakade, S.~M., Shalev-Shwartz, S., and Tewari, A.
\newblock Efficient bandit algorithms for online multiclass prediction.
\newblock In \emph{Proceedings of the 25th international conference on Machine
  learning}, pp.\  440--447, 2008.

\bibitem[Kleinberg et~al.(2008)Kleinberg, Slivkins, and
  Upfal]{kleinberg2008multi}
Kleinberg, R., Slivkins, A., and Upfal, E.
\newblock Multi-armed bandits in metric spaces.
\newblock In \emph{Proceedings of the fortieth annual ACM symposium on Theory
  of computing}, pp.\  681--690. ACM, 2008.

\bibitem[Krause \& Ong(2011)Krause and Ong]{krause2011contextual}
Krause, A. and Ong, C.~S.
\newblock Contextual {Gaussian} process bandit optimization.
\newblock In \emph{Advances in neural information processing systems}, pp.\
  2447--2455, 2011.

\bibitem[Kveton et~al.(2020)Kveton, Zaheer, Szepesvári, Li, Ghavamzadeh, and
  Boutilier]{kveton20randomized}
Kveton, B., Zaheer, M., Szepesvári, C., Li, L., Ghavamzadeh, M., and
  Boutilier, C.
\newblock Randomized exploration in generalized linear bandits.
\newblock In \emph{Proceedings of the 22nd International Conference on
  Artificial Intelligence and Statistics}, 2020.

\bibitem[Langford \& Zhang(2008)Langford and Zhang]{langford08epoch}
Langford, J. and Zhang, T.
\newblock The epoch-greedy algorithm for contextual multi-armed bandits.
\newblock In \emph{Advances in Neural Information Processing Systems 20
  (NIPS)}, pp.\  1096--1103, 2008.

\bibitem[Lattimore \& Szepesv\'{a}ri(2019)Lattimore and
  Szepesv\'{a}ri]{lattimore19bandit}
Lattimore, T. and Szepesv\'{a}ri, C.
\newblock \emph{Bandit Algorithms}.
\newblock Cambridge University Press, 2019.
\newblock In press.

\bibitem[LeCun et~al.(1998)LeCun, Bottou, Bengio, and
  Haffner]{lecun1998gradient}
LeCun, Y., Bottou, L., Bengio, Y., and Haffner, P.
\newblock Gradient-based learning applied to document recognition.
\newblock \emph{Proceedings of the IEEE}, 86\penalty0 (11):\penalty0
  2278--2324, 1998.

\bibitem[Li et~al.(2010)Li, Chu, Langford, and Schapire]{li2010contextual}
Li, L., Chu, W., Langford, J., and Schapire, R.~E.
\newblock A contextual-bandit approach to personalized news article
  recommendation.
\newblock In \emph{Proceedings of the 19th international conference on World
  wide web}, pp.\  661--670. ACM, 2010.

\bibitem[Li et~al.(2017)Li, Lu, and Zhou]{li2017provably}
Li, L., Lu, Y., and Zhou, D.
\newblock Provably optimal algorithms for generalized linear contextual
  bandits.
\newblock In \emph{Proceedings of the 34th International Conference on Machine
  Learning-Volume 70}, pp.\  2071--2080. JMLR. org, 2017.

\bibitem[Li \& Liang(2018)Li and Liang]{li2018learning}
Li, Y. and Liang, Y.
\newblock Learning overparameterized neural networks via stochastic gradient
  descent on structured data.
\newblock In \emph{Advances in Neural Information Processing Systems}, pp.\
  8157--8166, 2018.

\bibitem[Liang \& Srikant(2016)Liang and Srikant]{liang2016deep}
Liang, S. and Srikant, R.
\newblock Why deep neural networks for function approximation?
\newblock \emph{arXiv preprint arXiv:1610.04161}, 2016.

\bibitem[Lipton et~al.(2018)Lipton, Li, Gao, Li, Ahmed, and
  Deng]{lipton2018bbq}
Lipton, Z., Li, X., Gao, J., Li, L., Ahmed, F., and Deng, L.
\newblock {BBQ}-networks: Efficient exploration in deep reinforcement learning
  for task-oriented dialogue systems.
\newblock In \emph{Thirty-Second AAAI Conference on Artificial Intelligence},
  2018.

\bibitem[Lu et~al.(2017)Lu, Pu, Wang, Hu, and Wang]{lu2017expressive}
Lu, Z., Pu, H., Wang, F., Hu, Z., and Wang, L.
\newblock The expressive power of neural networks: A view from the width.
\newblock In \emph{Advances in neural information processing systems}, pp.\
  6231--6239, 2017.

\bibitem[Riquelme et~al.(2018)Riquelme, Tucker, and Snoek]{riquelme2018deep}
Riquelme, C., Tucker, G., and Snoek, J.
\newblock Deep {Bayesian} bandits showdown.
\newblock In \emph{International Conference on Learning Representations}, 2018.

\bibitem[Rusmevichientong \& Tsitsiklis(2010)Rusmevichientong and
  Tsitsiklis]{rusmevichientong2010linearly}
Rusmevichientong, P. and Tsitsiklis, J.~N.
\newblock Linearly parameterized bandits.
\newblock \emph{Mathematics of Operations Research}, 35\penalty0 (2):\penalty0
  395--411, 2010.

\bibitem[Russo et~al.(2018)Russo, Roy, Kazerouni, Osband, and
  Wen]{russo18tutorial}
Russo, D., Roy, B.~V., Kazerouni, A., Osband, I., and Wen, Z.
\newblock A tutorial on {Thompson} sampling.
\newblock \emph{Foundations and Trends in Machine Learning}, 11\penalty0
  (1):\penalty0 1--96, 2018.

\bibitem[Srinivas et~al.(2010)Srinivas, Krause, Kakade, and
  Seeger]{srinivas2009gaussian}
Srinivas, N., Krause, A., Kakade, S., and Seeger, M.
\newblock Gaussian process optimization in the bandit setting: no regret and
  experimental design.
\newblock In \emph{Proceedings of the 27th International Conference on
  International Conference on Machine Learning}, pp.\  1015--1022. Omnipress,
  2010.

\bibitem[Telgarsky(2015)]{telgarsky2015representation}
Telgarsky, M.
\newblock Representation benefits of deep feedforward networks.
\newblock \emph{arXiv preprint arXiv:1509.08101}, 2015.

\bibitem[Telgarsky(2016)]{telgarsky2016benefits}
Telgarsky, M.
\newblock Benefits of depth in neural networks.
\newblock \emph{arXiv preprint arXiv:1602.04485}, 2016.

\bibitem[Valko et~al.(2013)Valko, Korda, Munos, Flaounas, and
  Cristianini]{valko2013finite}
Valko, M., Korda, N., Munos, R., Flaounas, I., and Cristianini, N.
\newblock Finite-time analysis of kernelised contextual bandits.
\newblock \emph{arXiv preprint arXiv:1309.6869}, 2013.

\bibitem[Yang \& Wang(2019)Yang and Wang]{yang2019reinforcement}
Yang, L.~F. and Wang, M.
\newblock Reinforcement leaning in feature space: Matrix bandit, kernels, and
  regret bound.
\newblock \emph{arXiv preprint arXiv:1905.10389}, 2019.

\bibitem[Yarotsky(2017)]{yarotsky2017error}
Yarotsky, D.
\newblock Error bounds for approximations with deep {ReLU} networks.
\newblock \emph{Neural Networks}, 94:\penalty0 103--114, 2017.

\bibitem[Yarotsky(2018)]{yarotsky2018optimal}
Yarotsky, D.
\newblock Optimal approximation of continuous functions by very deep {ReLU}
  networks.
\newblock \emph{arXiv preprint arXiv:1802.03620}, 2018.

\bibitem[Zahavy \& Mannor(2019)Zahavy and Mannor]{zahavy2019deep}
Zahavy, T. and Mannor, S.
\newblock Deep neural linear bandits: Overcoming catastrophic forgetting
  through likelihood matching.
\newblock \emph{arXiv preprint arXiv:1901.08612}, 2019.

\bibitem[Zhang et~al.(2016)Zhang, Yang, Jin, Xiao, and Zhou]{zhang2016online}
Zhang, L., Yang, T., Jin, R., Xiao, Y., and Zhou, Z.-H.
\newblock Online stochastic linear optimization under one-bit feedback.
\newblock In \emph{International Conference on Machine Learning}, pp.\
  392--401, 2016.

\bibitem[Zou \& Gu(2019)Zou and Gu]{zou2019improved}
Zou, D. and Gu, Q.
\newblock An improved analysis of training over-parameterized deep neural
  networks.
\newblock In \emph{Advances in Neural Information Processing Systems}, 2019.

\bibitem[Zou et~al.(2019)Zou, Cao, Zhou, and Gu]{zou2018stochastic}
Zou, D., Cao, Y., Zhou, D., and Gu, Q.
\newblock Stochastic gradient descent optimizes over-parameterized deep {ReLU}
  networks.
\newblock \emph{Machine Learning}, 2019.

\end{thebibliography}
\bibliographystyle{icml2020}

\newpage
\onecolumn
\appendix
\section{Proof of Additional Results in Section \ref{sec:regret_analysis}}

\subsection{Verification of Remark \ref{remark:effect}}\label{sec:val}
Suppose there exists a mapping $\bpsi: \RR^d \rightarrow \RR^{\hat d}$ satisfying $\|\bpsi(\xb)\|_2 \leq 1$ which maps any context $\xb \in \RR^d$ to the Hilbert space $\cH$ associated with the Gram matrix $\Hb \in \RR^{TK \times TK}$ over contexts $\{\xb^i\}_{i=1}^{TK}$. Then $\Hb = \bPsi^\top\bPsi$, where $\bPsi = [\bpsi(\xb^1),\dots,\bpsi(\xb^{TK})] \in \RR^{\hat d \times TK}$. Thus, we can bound the effective dimension $\tilde d$ as follows
\begin{align*}
    \tilde d = \frac{\log\det [\Ib + \Hb/\lambda]}{\log (1+TK/\lambda)} = \frac{\log\det \big[\Ib + \bPsi\bPsi^\top/\lambda\big]}{\log (1+TK/\lambda)} \leq \hat d\cdot\frac{\log \big\|\Ib + \bPsi\bPsi^\top/\lambda\big\|_2}{\log (1+TK/\lambda)}\,.
\end{align*}
where the second equality holds due to the fact that $\det(\Ib + \Ab^\top\Ab/\lambda) = \det(\Ib + \Ab\Ab^\top/\lambda)$ holds for any matrix $\Ab$, and the inequality holds since $\det \Ab \leq \|\Ab\|_2^{\hat d}$ for any $\Ab \in \RR^{\hat d \times \hat d}$.
Clearly, $\tilde d \le \hat d$ as long as $\big\|\Ib + \bPsi\bPsi^\top/\lambda\big\|_2 \le 1+TK/\lambda$.  Indeed,
\begin{align*}
    \big\|\Ib + \bPsi\bPsi^\top/\lambda\big\|_2 \leq 1+\big\|\bPsi\bPsi^\top\big\|_2/\lambda \leq 1 + \sum_{i=1}^{TK}\big\|\bpsi(\xb^i)\bpsi(\xb^i)^\top\big\|_2/\lambda \leq 1+TK/\lambda\,,
\end{align*}
where the first inequality is due to triangle inequality and the fact $\lambda \geq 1$, the second inequality holds due to the definition of $\bPsi$ and triangle inequality, and the last inequality is by $\|\bpsi(\xb^i)\|_2 \leq 1$ for any $1 \leq i \leq TK$.

\subsection{Verification of Remark \ref{remark:rkhsnorm}}\label{sec:RKHSnorm}
\CC{
Let $K(\cdot, \cdot)$ be the NTK kernel, then for $i,j \in [TK]$, we have $\Hb_{i,j} = K(\xb^i, \xb^j)$.
Suppose that $h \in \cH$, then $h$ can be decomposed as $h = h_{\Hb} + h_\perp$, where $h_{\Hb}(\xb) = \sum_{i=1}^{TK} \alpha_i K(\xb, \xb^i)$ is the projection of $h$ to the function space spanned by $\{K(\xb, \xb^i)\}_{i=1}^{TK}$ and $h_\perp$ is the orthogonal part. By definition we have $h(\xb^i) = h_\Hb(\xb^i)$ for $i \in [TK]$, thus 
\begin{align}
    \hb &= [h(\xb^1),\dots,h(\xb^{TK})]^\top \notag \\
    &= [h_{\Hb}(\xb^1),\dots,h_{\Hb}(\xb^{TK})]^\top \notag \\
    &= \bigg[\sum_{i=1}^{TK} \alpha_i K(\xb^1, \xb^i),\dots,\sum_{i=1}^{TK} \alpha_i K(\xb^{TK}, \xb^i)\bigg]^\top\notag \\
    & = \Hb\balpha,\notag
\end{align}
which implies that $\balpha = \Hb^{-1}\hb$. Thus, we have
\begin{align}
    \|h\|_\cH \geq \|h_{\Hb}\|_\cH = \sqrt{\balpha^\top\Hb\balpha} = \sqrt{\hb^\top\Hb^{-1}\Hb \Hb^{-1}\hb} = \sqrt{\hb^\top\Hb^{-1}\hb}.\notag
\end{align}
}

\subsection{Proof of Corollary \ref{coro:expectation}}
\begin{proof}[Proof of Corollary \ref{coro:expectation}]
Notice that $R_T \leq T$ since $0 \leq h(\xb) \leq 1$. Thus, with the fact that with probability at least $1-\delta$, \eqref{corollary:regret-1} holds, we can bound $\EE [R_T]$ as
\begin{align}
    \EE [R_T] &\leq (1-\delta)\bigg(3\sqrt{T}\sqrt{\tilde d \log (1+TK/\lambda) + 2}\bigg[\nu\sqrt{\tilde d \log (1+TK/\lambda) + 2-2\log\delta}\notag \\
     &\qquad + 2\sqrt{\lambda}S +  (\lambda+C_2TL)(1- \eta m \lambda)^{J/2} \sqrt{T/\lambda}\bigg]+ 1\bigg) + \delta T. 
\end{align}
Taking $\delta = 1/T$ completes the proof.
\end{proof}

\section{Proof of Lemmas in Section \ref{section:newtheoryproof}}\label{section:new2}

\subsection{Proof of Lemma \ref{lemma:equal}}
We start with the following lemma:
\begin{lemma}\label{lemma:ntkconverge}
Let $\Gb = [\gb(\xb^1; \btheta_0),\ldots,\gb(\xb^{TK}; \btheta_0)]/\sqrt{m} \in \RR^{p \times (TK)}$. Let $\Hb$ be the NTK matrix as defined in Definition \ref{def:ntk}. For any $\delta \in(0,1)$, if
\begin{align}
    m  = \Omega\bigg(\frac{L^6\log (TKL/\delta)}{\epsilon^4}\bigg),\notag
\end{align}
then with probability at least $1-\delta$,
we have 
\begin{align}
    \|\Gb^\top\Gb - \Hb\|_F \leq TK\epsilon.\notag
\end{align}
\end{lemma}
We begin to prove Lemma \ref{lemma:equal}.
\begin{proof}[Proof of Lemma \ref{lemma:equal}]
By Assumption \ref{assumption:input}, we know that $\lambda_0>0$. By the choice of $m$, we have $m \geq \Omega(L^6 \log(TKL/\delta)/\epsilon^4)$, where $\epsilon = \lambda_0/(2TK)$. Thus, due to Lemma \ref{lemma:ntkconverge}, with probability at least $1-\delta$, we have $\|\Gb^\top\Gb - \Hb\|_F \leq TK\epsilon = \lambda_0/2$. That leads to
\begin{align}
    \Gb^\top\Gb \succeq \Hb - \|\Gb^\top\Gb - \Hb\|_F\Ib \succeq \Hb - \lambda_0\Ib/2 \succeq \Hb/2\succ 0,\label{eq:lemma:equal}
\end{align}
where the first inequality holds due to the triangle inequality, the third and fourth inequality holds due to $\Hb \succeq \lambda_0\Ib\succ0$. Thus, suppose the singular value decomposition of $\Gb$ is $\Gb = \Pb\Ab\Qb^\top$, $\Pb \in \RR^{p \times TK}, \Ab \in \RR^{TK \times TK}, \Qb \in \RR^{TK\times TK}$, we have $\Ab \succ 0$. Now we are going to show that $\btheta^* = \btheta_0 + \Pb\Ab^{-1}\Qb^\top\hb/\sqrt{m}$ satisfies \eqref{lemma:equal_0}. First, we have
\begin{align}
     \Gb^\top\sqrt{m}(\btheta^* - \btheta_0)& = \Qb\Ab\Pb^\top\Pb\Ab^{-1}\Qb^\top\hb = \hb,\notag
\end{align}
which suggests that for any $i$, $\la\gb(\xb^i; \btheta_0), \btheta^* - \btheta_0\ra = h(\xb^i)$. \CC{We also have
\begin{align}
    m\|\btheta^* - \btheta_0\|_2^2& =\hb^\top \Qb \Ab^{-2}\Qb^\top\hb = \hb^\top(\Gb^\top\Gb)^{-1}\hb \leq 2\hb^\top\Hb^{-1}\hb,\notag
\end{align}}
where the last inequality holds due to \eqref{eq:lemma:equal}.
This completes the proof. 
\end{proof}

\subsection{Proof of Lemma \ref{lemma:newinset}}
In this section we prove Lemma \ref{lemma:newinset}.
For simplicity, we define $\bar \Zb_t, \bar \bbb_t, \bar\gamma_t$ as follows:
\begin{align}
        \bar\Zb_t &= \lambda \Ib + \sum_{i=1}^t\gb(\xb_{i,a_i}; \btheta_{0})\gb(\xb_{i,a_i}; \btheta_{0})^\top/m,\notag \\
        \bar\bbb_t &= \sum_{i=1}^t r_{i,a_i} \gb(\xb_{i,a_i}; \btheta_{0})/\sqrt{m} ,\notag \\
        \bar\gamma_t &= \nu\sqrt{\log\frac{\det \bar\Zb_t}{\det \lambda \Ib}-2\log\delta} + \sqrt{\lambda}S.\notag
\end{align}

We need the following lemmas. The first lemma shows that the network parameter $\btheta_t$ at round $t$ can be well approximated by $\btheta_0 + \bar \Zb_t^{-1}\bar\bbb_t/\sqrt{m}$. 
\begin{lemma}\label{lemma:newboundreference}
There exist constants $\{\bar C_i\}_{i=1}^5>0$ such that for any $\delta > 0$, if for all $t \in [T]$, $\eta, m$ satisfy
\begin{align}
    &2\sqrt{t/(m\lambda)} \geq \bar C_1m^{-3/2}L^{-3/2}[\log(TKL^2/\delta)]^{3/2},\notag \\
    &2\sqrt{t/(m\lambda)} \leq \bar C_2\min\big\{ L^{-6}[\log m]^{-3/2},\big(m(\lambda\eta)^2L^{-6}t^{-1}(\log m)^{-1}\big)^{3/8} \big\},\notag \\
    &\eta \leq \bar C_3(m\lambda + tmL)^{-1},\notag \\
    &m^{1/6}\geq \bar C_4\sqrt{\log m}L^{7/2}t^{7/6}\lambda^{-7/6}(1+\sqrt{t/\lambda}),\notag
\end{align}
then with probability at least $1-\delta$,
we have that $\|\btheta_t -\btheta_0\|_2 \leq  2\sqrt{t/(m\lambda )}$ and
\begin{align}
   \|\btheta_t - \btheta_0 - \bar \Zb_t^{-1}\bar\bbb_t/\sqrt{m}\|_2  \leq (1- \eta m \lambda)^{J/2} \sqrt{t/(m\lambda)} + \bar C_5m^{-2/3}\sqrt{\log m}L^{7/2}t^{5/3}\lambda^{-5/3}(1+\sqrt{t/\lambda}).\notag
\end{align}
\end{lemma}
Next lemma shows the error bounds for $\bar \Zb_t$ and $\Zb_t$. 
\begin{lemma}\label{lemma:newboundz}
There exist constants $\{\bar C_i\}_{i=1}^5>0$ such that for any $\delta > 0$, if $m$ satisfies that
\begin{align}
    \bar C_1m^{-3/2}L^{-3/2}[\log(TKL^2/\delta)]^{3/2}\leq2\sqrt{t/(m\lambda)} \leq \bar C_2 L^{-6}[\log m]^{-3/2},\ \forall t \in [T],\notag
\end{align}
then with probability at least $1-\delta$,
for any $t \in [T]$, we have
\begin{align}
& \|\Zb_t\|_2 \leq \lambda + \bar C_3tL,\notag\\
    &\|\bar\Zb_t -  \Zb_t\|_F \leq \bar C_4m^{-1/6}\sqrt{\log m}L^4t^{7/6}\lambda^{-1/6},\notag\\
   & \bigg|\log \frac{\det(\bar \Zb_t)}{\det(\lambda\Ib)} - \log\frac{ \det(\Zb_t)}{\det (\lambda\Ib)}\bigg| \leq \bar C_5m^{-1/6}\sqrt{\log m} L^4t^{5/3}\lambda^{-1/6}.\notag
\end{align}
\end{lemma}
With above lemmas, we prove Lemma \ref{lemma:newinset} as follows.
\begin{proof}[Proof of Lemma \ref{lemma:newinset}]
By Lemma \ref{lemma:newboundreference} we know that $\|\btheta_t- \btheta_0\|_2\leq 2\sqrt{t/(m\lambda)}$. By Lemma \ref{lemma:equal}, with probability at least $1-\delta$, there exists $\btheta^*$ such that for any $1 \leq t \leq T$,
\begin{align}
    &h(\xb_{t,a_t}) = \la \gb(\xb_{t,a_t}; \btheta_0)/\sqrt{m}, ~ \sqrt{m}(\btheta^* - \btheta_0)\ra,\label{eq:52_1}\\
    &\sqrt{m}\|\btheta^* - \btheta_0\|_2 \leq \sqrt{2\hb^\top\Hb^{-1}\hb} \leq S,\label{eq:52_2}
\end{align}
where the second inequality holds since $S \geq \sqrt{2\hb^\top \Hb^{-1}\hb}$ in the statement of Lemma \ref{lemma:newinset}.
\CC{Thus, conditioned on \eqref{eq:52_1} and \eqref{eq:52_2}, by Theorem 2 in \citet{abbasi2011improved}}, with probability at least $1-\delta$, for any $1 \leq t \leq T$, $\btheta^*$ satisfies that
\begin{align}
    \|\sqrt{m}(\btheta^* - \btheta_0) - \bar \Zb_t^{-1}\bar\bbb_t\|_{\bar \Zb_t} \leq \bar\gamma_t.\label{eq:newinset_0}
\end{align}
We now prove that $\|\btheta^* - \btheta_t\|_{ \Zb_t} \leq \gamma_t/\sqrt{m}$.  From the triangle inequality,
\begin{align}
    \|\btheta^* - \btheta_t\|_{ \Zb_t} &\leq \underbrace{\|\btheta^* - \btheta_0 - \bar \Zb_t^{-1}\bar\bbb_t/\sqrt{m}\|_{ \Zb_t}}_{I_1} + \underbrace{\|\btheta_t - \btheta_0 - \bar \Zb_t^{-1}\bar\bbb_t/\sqrt{m}\|_{ \Zb_t}}_{I_2} \,. \label{eq:newinset_1}
\end{align}
We bound $I_1$ and $I_2$ separately. For $I_1$, we have
\begin{align}
    I_1^2 &= (\btheta^* - \btheta_0 -  \bar\Zb_t^{-1}\bar\bbb_t/\sqrt{m})^\top \Zb_t(\btheta^* - \btheta_0 - \bar\Zb_t^{-1}\bar\bbb_t/\sqrt{m})\notag \\
    & = (\btheta^* - \btheta_0 -  \bar\Zb_t^{-1}\bar\bbb_t/\sqrt{m})^\top \bar \Zb_t(\btheta^* - \btheta_0 -  \bar\Zb_t^{-1}\bar\bbb_t/\sqrt{m})\notag \\
    &\qquad + (\btheta^* - \btheta_0 -  \bar\Zb_t^{-1}\bar\bbb_t/\sqrt{m})^\top (\Zb_t - \bar \Zb_t)(\btheta^* - \btheta_0 -  \bar\Zb_t^{-1}\bar\bbb_t/\sqrt{m})\notag\\
    & \leq (\btheta^* - \btheta_0 -  \bar\Zb_t^{-1}\bar\bbb_t/\sqrt{m})^\top \bar \Zb_t(\btheta^* - \btheta_0 -  \bar\Zb_t^{-1}\bar\bbb_t/\sqrt{m})\notag \\
    &\qquad + \frac{\|\Zb_t - \bar\Zb_t\|_2}{\lambda}(\btheta^* - \btheta_0 -  \bar\Zb_t^{-1}\bar\bbb_t/\sqrt{m})^\top \bar \Zb_t(\btheta^* - \btheta_0 -  \bar\Zb_t^{-1}\bar\bbb_t/\sqrt{m})\notag \\
    &  \leq (1+ \|\Zb_t - \bar\Zb_t\|_2/\lambda)\bar\gamma_t^2/m,\label{eq:newinset_2}
\end{align}
where the first inequality holds due to the fact that $\xb^\top\Ab\xb \leq \xb^\top\Bb\xb\cdot \|\Ab\|_2/\lambda_{\text{min}}(\Bb)$ for some $\Bb \succ 0$ and the fact that $\lambda_{\text{min}}(\bar\Zb_t) \geq \lambda$, the second inequality holds due to \eqref{eq:newinset_0}. We have
\begin{align}
    \big\|\bar\Zb_t -  \Zb_t\big\|_2\leq \big\|\bar\Zb_t -  \Zb_t\big\|_F\leq  C_1m^{-1/6}\sqrt{\log m}L^4t^{7/6}\lambda^{-1/6},\label{eq:newinset_2.1}
\end{align}
where the first inequality holds due to the fact that $\|\Ab\|_2 \leq \|\Ab\|_F$, the second inequality holds due to Lemma \ref{lemma:newboundz}. We also have
\begin{align}
    \bar\gamma_t &= \nu\sqrt{\log\frac{\det \bar\Zb_t}{\det \lambda \Ib}-2\log\delta} + \sqrt{\lambda}S \notag \\
    & = 
    \nu\sqrt{\log\frac{\det \Zb_t}{\det \lambda \Ib} + \log\frac{\det \bar\Zb_t}{\det \lambda \Ib} - \log\frac{\det \Zb_t}{\det \lambda \Ib}-2\log\delta} + \sqrt{\lambda}S \notag \\
    & \leq \nu\sqrt{\log\frac{\det \Zb_t}{\det \lambda \Ib} + C_2m^{-1/6}\sqrt{\log m} L^4t^{5/3}\lambda^{-1/6}-2\log\delta} + \sqrt{\lambda}S,\label{eq:newinset_2.2}
\end{align}
where $C_1, C_2>0$ are two constants, the inequality holds due to Lemma \ref{lemma:newboundz}. Substituting \eqref{eq:newinset_2.1} and \eqref{eq:newinset_2.2} into \eqref{eq:newinset_2}, we have
\begin{align}
    I_1 &\leq \sqrt{1+ \|\Zb_t - \bar\Zb_t\|_2/\lambda}\bar\gamma_t/\sqrt{m}\notag \\
    & \leq \sqrt{1+C_1m^{-1/6}\sqrt{\log m}L^4t^{7/6}\lambda^{-7/6}}/\sqrt{m}\notag \\
    &\qquad \cdot \left(\nu\sqrt{\log\frac{\det \Zb_t}{\det \lambda \Ib} + C_2m^{-1/6}\sqrt{\log m} L^4t^{5/3}\lambda^{-1/6}-2\log\delta} + \sqrt{\lambda}S\right).\label{eq:newinset_4}
\end{align}
For $I_2$, we have
\begin{align}
    I_2& = 
    \|\btheta_t - \btheta_0 - \bar \Zb_t^{-1}\bar\bbb_t/\sqrt{m}\|_{ \Zb_t}\notag \\
    & \leq 
    \|\Zb_t\|_2\cdot\|\btheta_t - \btheta_0 - \bar \Zb_t^{-1}\bar\bbb_t/\sqrt{m}\|_2\notag \\
    &\leq (\lambda + C_3 t L)\|\btheta_t - \btheta_0 - \bar \Zb_t^{-1}\bar\bbb_t/\sqrt{m}\|_2 \notag \\
    & \leq (\lambda + C_3tL)\Big[(1- \eta m \lambda)^{J/2} \sqrt{t/(m\lambda)} + m^{-2/3}\sqrt{\log m}L^{7/2}t^{5/3}\lambda^{-5/3}(1+\sqrt{t/\lambda})\Big],\label{eq:newinset_5}
\end{align}
where $C_3>0$ is a constant, the first inequality  holds since for any vector $\ab$, the second inequality holds due to $\|\Zb_t\|_2 \leq \lambda + C_3 tL$ by Lemma \ref{lemma:newboundz}, the third inequality holds due to Lemma \ref{lemma:newboundreference}. Substituting \eqref{eq:newinset_4} and \eqref{eq:newinset_5} into \eqref{eq:newinset_1}, we obtain $\big\|\btheta^* - \btheta_t\big\|_{\Zb_t} \leq \gamma_t/\sqrt{m}$.  This completes the proof. 
\end{proof}

\subsection{Proof of Lemma \ref{lemma:newboundh}}
The proof starts with three lemmas that bound the error terms of the function value and gradient of neural networks. 
\begin{lemma}[Lemma 4.1, \citet{cao2019generalization2}]\label{lemma:cao_functionvalue}
There exist constants $\{\bar C_i\}_{i=1}^3 >0$ such that for any $\delta > 0$, if $\tau$ satisfies that
\begin{align}
    \bar C_1m^{-3/2}L^{-3/2}[\log(TKL^2/\delta)]^{3/2}\leq\tau \leq \bar C_2 L^{-6}[\log m]^{-3/2},\notag
\end{align}
then with probability at least $1-\delta$,
for all $\tilde\btheta, \hat\btheta$ satisfying $ \|\tilde\btheta - \btheta_0\|_2 \leq \tau, \|\hat\btheta - \btheta_0\|_2 \leq \tau$ and $j \in [TK]$ we have
\begin{align}
    \Big|f(\xb^j; \tilde\btheta) - f(\xb^j;  \hat\btheta) - \la \gb(\xb^j; 
    \hat\btheta),\tilde\btheta - \hat\btheta\ra\Big| \leq \bar C_3\tau^{4/3}L^3\sqrt{m \log m}.\notag
\end{align}
\end{lemma}
\begin{lemma}[Theorem 5, \citet{allen2018convergence}]\label{lemma:cao_gradientdifference}
There exist constants $\{\bar C_i\}_{i=1}^3>0$ such that for any $\delta \in (0,1)$, if $\tau$ satisfies that
\begin{align}
    \bar C_1 m^{-3/2}L^{-3/2}\max \{\log^{-3/2}m, \log^{3/2} (TK/\delta) \}\leq\tau \leq \bar C_2 L^{-9/2}\log^{-3}m,\notag
\end{align}
then with probability at least $1-\delta$,
for all $\|\btheta - \btheta_0\|_2 \leq \tau$ and $j \in [TK]$ we have
\begin{align}
  \| \gb(\xb^j; \btheta) - \gb(\xb^j; \btheta_0)\|_2 \leq \bar C_3\sqrt{\log m}\tau^{1/3}L^3\|\gb(\xb^j;  \btheta_0)\|_2.\notag
\end{align}
\end{lemma}
\begin{lemma}[Lemma B.3, \citet{cao2019generalization2}]\label{lemma:cao_boundgradient}
There exist constants $\{\bar C_i\}_{i=1}^3>0$ such that for any $\delta > 0$, if $\tau$ satisfies that
\begin{align}
    \bar C_1m^{-3/2}L^{-3/2}[\log(TKL^2/\delta)]^{3/2}\leq\tau \leq \bar C_2 L^{-6}[\log m]^{-3/2},\notag
\end{align}
then with probability at least $1-\delta$,
for any $\|\btheta - \btheta_0\|_2 \leq \tau$ and $j \in[TK]$ 
we have $\|\gb(\xb^j; \btheta)\|_F\leq \bar C_3\sqrt{mL}$.
\end{lemma}

\begin{proof}[Proof of Lemma \ref{lemma:newboundh}]
\CC{We follow the regret bound analysis in \citet{abbasi2011improved, valko2013finite}}. Denote $a_t^* = \argmax_{a \in [K]}h(\xb_{t, a})$ and $\cC_t = \{\btheta: \|\btheta - \btheta_{t}\|_{\Zb_{t}} \leq \gamma_{t}/\sqrt{m}\}$. 
By Lemma \ref{lemma:newinset}, for all $1 \leq t \leq T$, we have $\|\btheta_t - \btheta_0\|_2 \leq 2\sqrt{t/(m\lambda)}$ and $\btheta^*  \in \cC_t$. By the choice of $m$, Lemmas \ref{lemma:cao_functionvalue}, \ref{lemma:cao_gradientdifference} and \ref{lemma:cao_boundgradient} hold. Thus, $h(\xb_{t,a_t^*}) - h(\xb_{t,a_t})$ can be bounded as follows:
\begin{align}
     &h(\xb_{t,a_t^*}) - h(\xb_{t,a_t}) \notag \\
     & = \la \gb(\xb_{t,a_t^*}; \btheta_0), \btheta^* - \btheta_0\ra - \la \gb(\xb_{t,a_t}; \btheta_0), \btheta^* - \btheta_0\ra\notag \\
     & \leq \la \gb(\xb_{t,a_t^*}; \btheta_{t-1}), \btheta^* - \btheta_0\ra - \la \gb(\xb_{t,a_t}; \btheta_{t-1}), \btheta^* - \btheta_0\ra \notag \\
     &\qquad + \|\btheta^* - \btheta_0\|_2(\|\gb(\xb_{t,a_t^*}; \btheta_{t-1}) - \gb(\xb_{t,a_t^*}; \btheta_{0})\|_2 + \|\gb(\xb_{t,a_t}; \btheta_{t-1}) - \gb(\xb_{t,a_t}; \btheta_{0})\|_2)\notag \\
     & \leq \la \gb(\xb_{t,a_t^*}; \btheta_{t-1}), \btheta^* - \btheta_0\ra - \la \gb(\xb_{t,a_t}; \btheta_{t-1}), \btheta^* - \btheta_0\ra  + C_1\sqrt{\hb^\top\Hb^{-1}\hb}m^{-1/6}\sqrt{\log m}t^{1/6}\lambda^{-1/6}L^{7/2}\notag \\
     & \leq \underbrace{\max_{\btheta \in \cC_{t-1}}\la \gb(\xb_{t,a_t^*}; \btheta_{t-1}), \btheta - \btheta_0\ra - \la \gb(\xb_{t,a_t}; \btheta_{t-1}), \btheta^* - \btheta_0\ra}_{I_1}  + C_1\sqrt{\hb^\top\Hb^{-1}\hb}m^{-1/6}\sqrt{\log m}t^{1/6}\lambda^{-1/6}L^{7/2} ,\label{eq:newalgorithm_0}
\end{align}
where the equality holds due to Lemma \ref{lemma:equal}, the first inequality holds due to triangle inequality, the second inequality holds due to Lemmas \ref{lemma:equal}, \ref{lemma:cao_gradientdifference}, \ref{lemma:cao_boundgradient}, the third inequality holds due to $\btheta^* \in \cC_{t-1}$. Denote 
\begin{align}
    \tilde U_{t,a} = \la \gb(\xb_{t,a}; \btheta_{t-1}), \btheta_{t-1} - \btheta_0\ra  + \gamma_{t-1}\sqrt{\gb(\xb_{t,a}; \btheta_{t-1})^\top\Zb_{t-1}^{-1}\gb(\xb_{t,a}; \btheta_{t-1})/m},\notag
\end{align}
then we have $\tilde U_{t,a} = \max_{\btheta \in \cC_{t-1}} \la \gb(\xb_{t,a}; \btheta_{t-1}), \btheta - \btheta_0\ra$ due to the fact that
\begin{align}
    \max_{\xb:\|\xb - \bbb\|_{\Ab} \leq c} \la \ab, \xb\ra = \la \ab, \bbb\ra + c\sqrt{\ab^\top\Ab^{-1}\ab} .\notag
\end{align}
Recall the definition of $U_{t,a}$ from Algorithm \ref{algorithm:2}, we also have
\begin{align}
    |U_{t,a} - \tilde U_{t,a}|  &= \big|f(\xb_{t,a}; \btheta_{t-1}) - \la \gb(\xb_{t,a}; \btheta_{t-1}), \btheta_{t-1} - \btheta_0\ra \big|\notag \\
    & = \big|f(\xb_{t,a}; \btheta_{t-1}) - f(\xb_{t,a}; \btheta_{0}) - \la \gb(\xb_{t,a}; \btheta_{t-1}), \btheta_{t-1} - \btheta_0\ra \big|\notag \\
    &\leq C_2m^{-1/6}\sqrt{\log m} t^{2/3}\lambda^{-2/3}L^3,\label{eq:newalgorithm_0.1}
\end{align}
where $C_2>0$ is a constant, the second equality holds due to $f(\xb^j;\btheta_0) = 0$ by the random initialization of $\btheta_0$, the inequality holds due to Lemma \ref{lemma:cao_functionvalue} with the fact $\|\btheta_{t-1} -\btheta_0\|_2\leq 2\sqrt{t/(m\lambda)})$.
Since $\btheta^* \in \cC_{t-1}$, then $I_1$ in \eqref{eq:newalgorithm_0} can be bounded as
\begin{align}
    &\max_{\btheta \in \cC_{t-1}}\la \gb(\xb_{t,a_t^*}; \btheta_{t-1}), \btheta - \btheta_0\ra - \la \gb(\xb_{t,a_t}; \btheta_{t-1}), \btheta^* - \btheta_0\ra \notag \\
    & =  \tilde U_{t, a_t^*} - \la \gb(\xb_{t,a_t}; \btheta_{t-1}), \btheta^* - \btheta_0\ra \notag \\
    &\leq  U_{t, a_t^*}-\la \gb(\xb_{t,a_t}; \btheta_{t-1}), \btheta^* - \btheta_0\ra + C_2m^{-1/6}\sqrt{\log m} t^{2/3}\lambda^{-2/3}L^3\notag \\
    & \leq U_{t, a_t}-\la \gb(\xb_{t,a_t}; \btheta_{t-1}), \btheta^* - \btheta_0\ra + C_2m^{-1/6}\sqrt{\log m} t^{2/3}\lambda^{-2/3}L^3\notag \\
    & \leq \tilde U_{t, a_t}-\la \gb(\xb_{t,a_t}; \btheta_{t-1}), \btheta^* - \btheta_0\ra + 2C_2m^{-1/6}\sqrt{\log m} t^{2/3}\lambda^{-2/3}L^3,\label{eq:newalgorithm_2}
\end{align}
where the first inequality holds due to \eqref{eq:newalgorithm_0.1}, the second inequality holds since $a_t = \argmax_{a} U_{t,a}$, the third inequality holds due to \eqref{eq:newalgorithm_0.1}. Furthermore, 
\begin{align}
   &\tilde U_{t, a_t}-\la \gb(\xb_{t,a_t}; \btheta_{t-1}), \btheta^* - \btheta_0\ra \notag \\
   &= \max_{\btheta \in \cC_{t-1}} \la \gb(\xb_{t,a_t}; \btheta_{t-1}), \btheta - \btheta_0\ra - \la \gb(\xb_{t,a_t}; \btheta_{t-1}), \btheta^* - \btheta_0\ra\notag \\
   & = \max_{\btheta \in \cC_{t-1}} \la \gb(\xb_{t,a_t}; \btheta_{t-1}), \btheta - \btheta_{t-1}\ra - \la \gb(\xb_{t,a_t}; \btheta_{t-1}), \btheta^* - \btheta_{t-1}\ra\notag\\
   & \leq \max_{\btheta \in \cC_{t-1}}\big\|\btheta - \btheta_{t-1}\big\|_{\Zb_{t-1}} \|\gb(\xb_{t,a_t}; \btheta_{t-1})\|_{\Zb_{t-1}^{-1}} + \big\|\btheta^* - \btheta_{t-1}\big\|_{\Zb_{t-1}} \|\gb(\xb_{t,a_t}; \btheta_{t-1})\|_{\Zb_{t-1}^{-1}}\notag \\
   & \leq 2\gamma_{t-1}\|\gb(\xb_{t,a_t}; \btheta_{t-1})/\sqrt{m}\|_{\Zb_{t-1}^{-1}},\label{eq:newalgorithm_3}
\end{align}
where the first inequality holds due to H\"{o}lder inequality, the second inequality holds due to Lemma \ref{lemma:newinset}. Combining \eqref{eq:newalgorithm_0}, \eqref{eq:newalgorithm_2} and \eqref{eq:newalgorithm_3}, we have
\begin{align}
    &h(\xb_{t,a_t^*}) - h(\xb_{t,a_t})\notag \\
    & \leq 2\gamma_{t-1}\|\gb(\xb_{t,a_t}; \btheta_{t-1})/\sqrt{m}\|_{\Zb_{t-1}^{-1}} + C_1\sqrt{\hb^\top\Hb^{-1}\hb}m^{-1/6}\sqrt{\log m}t^{1/6}\lambda^{-1/6}L^{7/2}\notag \\
    & \qquad + 2C_2m^{-1/6}\sqrt{\log m} t^{2/3}\lambda^{-2/3}L^3\notag \\
    & \leq \min\bigg\{2\gamma_{t-1}\|\gb(\xb_{t,a_t}; \btheta_{t-1})/\sqrt{m}\|_{\Zb_{t-1}^{-1}} + C_1\sqrt{\hb^\top\Hb^{-1}\hb}m^{-1/6}\sqrt{\log m}t^{1/6}\lambda^{-1/6}L^{7/2}\notag \\
    & \qquad + 2C_2m^{-1/6}\sqrt{\log m} t^{2/3}\lambda^{-2/3}L^3, 1\bigg\}\notag \\
    & \leq \min\bigg\{2\gamma_{t-1}\|\gb(\xb_{t,a_t}; \btheta_{t-1})/\sqrt{m}\|_{\Zb_{t-1}^{-1}}, 1\bigg\} + C_1\sqrt{\hb^\top\Hb^{-1}\hb}m^{-1/6}\sqrt{\log m}t^{1/6}\lambda^{-1/6}L^{7/2}\notag \\
    & \qquad + 2C_2m^{-1/6}\sqrt{\log m} t^{2/3}\lambda^{-2/3}L^3\notag \\
    & \leq 2\gamma_{t-1}\min\bigg\{\|\gb(\xb_{t,a_t}; \btheta_{t-1})/\sqrt{m}\|_{\Zb_{t-1}^{-1}}, 1\bigg\} + C_1\sqrt{\hb^\top\Hb^{-1}\hb}m^{-1/6}\sqrt{\log m}t^{1/6}\lambda^{-1/6}L^{7/2}\notag \\
    & \qquad + 2C_2m^{-1/6}\sqrt{\log m} t^{2/3}\lambda^{-2/3}L^3,\label{eq:newalgorithm_4}
\end{align}
where the second inequality holds due to the fact that 
 $0 \leq h(\xb_{t,a_t^*}) - h(\xb_{t,a_t}) \leq 1$, the third inequality holds due to the fact that $\min\{a+b,1\} \leq \min\{a,1\}+b$, the fourth inequality holds due to the fact $\gamma_{t-1} \geq \sqrt{\lambda}S \geq 1$. Finally, by the fact that $\sqrt{2\hb\Hb^{-1}\hb} \leq S$, the proof completes.
\end{proof}

\subsection{Proof of Lemma \ref{lemma:newboundregret}}
In this section we prove Lemma \ref{lemma:newboundregret}, we need the following lemma from \citet{abbasi2011improved}.
\begin{lemma}[Lemma 11, \citet{abbasi2011improved}]\label{lemma:oriself}
We have the following inequality:
\begin{align}
    \sum_{t=1}^T \min\bigg\{\|\gb(\xb_{t,a_t}; \btheta_{t-1})/\sqrt{m}\|_{\Zb_{t-1}^{-1}}^2,1\bigg\} \leq 2 \log \frac{\det \Zb_T}{\det \lambda \Ib}.\notag
\end{align}
\end{lemma}

\begin{proof}[Proof of Lemma \ref{lemma:newboundregret}]
First by the definition of $\gamma_t$, we know that $\gamma_t$ is a monotonic function w.r.t. $\det \Zb_t$. By the definition of $\Zb_t$, we know that $\Zb_T \succeq \Zb_t$, which implies that $\det \Zb_t \leq \det \Zb_T$. Thus, $\gamma_t \leq \gamma_T$. 
Second, by Lemma \ref{lemma:oriself} we know that 
\begin{align}
     &\sum_{t=1}^T \min\bigg\{\|\gb(\xb_{t,a_t}; \btheta_{t-1})/\sqrt{m}\|_{\Zb_{t-1}^{-1}}^2,1\bigg\} \notag \\
     &\leq 2 \log \frac{\det \Zb_T}{\det \lambda \Ib}\notag \\
     & \leq 2 \log \frac{\det \bar \Zb_T}{\det \lambda \Ib} + C_1m^{-1/6}\sqrt{\log m} L^4T^{5/3}\lambda^{-1/6},\label{selfnormal:1}
\end{align}
where the second inequality holds due to Lemma \ref{lemma:newboundz}. Next we are going to bound $\log \det \bar\Zb_T$. Denote $\Gb = [\gb(\xb^1; \btheta_0)/\sqrt{m},\dots,\gb(\xb^{TK}; \btheta_0)/\sqrt{m}] \in \RR^{p \times (TK)}$, then we have
\begin{align}
    \log \frac{\det \bar \Zb_T}{\det \lambda \Ib} 
    & = \log \det \bigg(\Ib + \sum_{t=1}^T \gb(\xb_{t,a_t}; \btheta_0)\gb(\xb_{t,a_t}; \btheta_0)^\top/(m\lambda) \bigg)\notag \\
    & \leq \log \det \bigg(\Ib + \sum_{i=1}^{TK} \gb(\xb^i; \btheta_0)\gb(\xb^i; \btheta_0)^\top/(m\lambda) \bigg)\notag \\
    & = \log \det \bigg(\Ib + \Gb\Gb^\top/\lambda \bigg)\notag \\
    & = \log \det \bigg(\Ib + \Gb^\top\Gb/\lambda \bigg),\label{selfnormal:2}
\end{align}
where the inequality holds naively, the third equality holds since for any matrix $\Ab \in \RR^{p \times TK}$, we have $\det (\Ib + \Ab\Ab^\top) = \det (\Ib + \Ab^\top\Ab)$.  We can further bound \eqref{selfnormal:2} as follows:
\CC{
\begin{align}
    \log \det \bigg(\Ib + \Gb^\top\Gb/\lambda \bigg) &= \log \det \bigg(\Ib +\Hb/\lambda + (\Gb^\top\Gb - \Hb)/\lambda \bigg)\notag \\
    & \leq \log \det \bigg(\Ib +\Hb/\lambda\bigg) + \la (\Ib + \Hb/\lambda)^{-1},(\Gb^\top\Gb - \Hb)/\lambda \ra \notag \\
    & \leq \log \det \bigg(\Ib +\Hb/\lambda\bigg) + \|(\Ib + \Hb/\lambda)^{-1}\|_F\|\Gb^\top\Gb - \Hb\|_F/\lambda\notag \\
    & \leq \log \det \bigg(\Ib +\Hb/\lambda\bigg) + \sqrt{TK}\|\Gb^\top\Gb - \Hb\|_F\notag \\
    & \leq \log \det \bigg(\Ib +\Hb/\lambda\bigg) +1\notag \\
    & = \tilde d\log(1+TK/\lambda) +1,\label{selfnormal:3}
\end{align}
}
where the first inequality holds due to the concavity of $\log \det (\cdot)$, the second inequality holds due to the fact that $\la \Ab, \Bb\ra \leq \|\Ab\|_F \|\Bb\|_F$, the third inequality holds due to the facts that $\Ib + \Hb/\lambda \succeq \Ib$, $\lambda \geq 1$ and $\|\Ab\|_F \leq \sqrt{TK}\|\Ab\|_2$ for any $\Ab \in \RR^{TK \times TK}$, the fourth inequality holds by Lemma \ref{lemma:ntkconverge} with the choice of $m$, 
the fifth inequality holds by the definition of effective dimension in Definition \ref{def:effective_dimension}, and the last inequality holds due to the choice of $\lambda$. Substituting \eqref{selfnormal:3} into \eqref{selfnormal:2}, we obtain that
\begin{align}
     \log \frac{\det \bar \Zb_T}{\det \lambda \Ib} \leq \tilde d\log(1+TK/\lambda) +1.\label{selfnormal:3.1}
\end{align}
Substituting \eqref{selfnormal:3.1} into \eqref{selfnormal:1}, we have
\begin{align}
    \sum_{t=1}^T \min\bigg\{\|\gb(\xb_{t,a_t}; \btheta_{t-1})/\sqrt{m}\|_{\Zb_{t-1}^{-1}}^2,1\bigg\} \leq 2\tilde d\log(1+TK/\lambda) +2 + C_1m^{-1/6}\sqrt{\log m} L^4T^{5/3}\lambda^{-1/6}. \label{selfnormal:3.2}
\end{align}
We now bound $\gamma_T$, which is 
\begin{align}
    \gamma_T &= \sqrt{1+C_1m^{-1/6}\sqrt{\log m}L^4T^{7/6}\lambda^{-7/6}}\notag \\
    &\qquad \cdot \bigg(\nu\sqrt{\log\frac{\det \Zb_T}{\det \lambda \Ib} + C_2m^{-1/6}\sqrt{\log m} L^4T^{5/3}\lambda^{-1/6}-2\log\delta} + \sqrt{\lambda}S\bigg)\notag \\
    &\qquad +  (\lambda +C_3TL)\Big[(1- \eta m \lambda)^{J/2} \sqrt{T/(m\lambda)} + m^{-2/3}\sqrt{\log m}L^{7/2}T^{5/3}\lambda^{-5/3}(1+\sqrt{T/\lambda})\Big]\notag\\
    & \leq \sqrt{1+C_1m^{-1/6}\sqrt{\log m}L^4T^{7/6}\lambda^{-7/6}}\notag \\
    &\qquad \cdot \bigg(\nu\sqrt{\log\frac{\det \bar \Zb_T}{\det \lambda \Ib} + 2C_2m^{-1/6}\sqrt{\log m} L^4T^{5/3}\lambda^{-1/6}-2\log\delta} + \sqrt{\lambda}S\bigg)\notag \\
    &\qquad +  (\lambda +C_3TL)\Big[(1- \eta m \lambda)^{J/2} \sqrt{T/(m\lambda)} + m^{-2/3}\sqrt{\log m}L^{7/2}T^{5/3}\lambda^{-5/3}(1+\sqrt{T/\lambda})\Big],\label{selfnormal:100}
\end{align}
where the inequality holds due to Lemma \ref{lemma:newboundz}. Finally, we have
\begin{align}
     &\sqrt{\sum_{t=1}^T\gamma_{t-1}^2\min\bigg\{\|\gb(\xb_{t,a_t}; \btheta_{t-1})/\sqrt{m}\|_{\Zb_{t-1}^{-1}}^2, 1\bigg\}} \notag \\
     & \leq \gamma_T\sqrt{\sum_{t=1}^T\min\bigg\{\|\gb(\xb_{t,a_t}; \btheta_{t-1})/\sqrt{m}\|_{\Zb_{t-1}^{-1}}^2, 1\bigg\}}\notag \\
     &\leq  \sqrt{\log\frac{\det \bar\Zb_T}{\det \lambda\Ib} + C_1m^{-1/6}\sqrt{\log m} L^4T^{5/3}\lambda^{-1/6}}\bigg[\sqrt{1+C_1m^{-1/6}\sqrt{\log m}L^4T^{7/6}\lambda^{-7/6}}\notag \\
    &\qquad \cdot \bigg(\nu\sqrt{\log\frac{\det \bar \Zb_T}{\det \lambda \Ib} + 2C_2m^{-1/6}\sqrt{\log m} L^4T^{5/3}\lambda^{-1/6}-2\log\delta} + \sqrt{\lambda}S\bigg)\notag \\
    &\qquad +  (\lambda +C_3TL)\Big[(1- \eta m \lambda)^{J/2} \sqrt{T/(m\lambda)} + m^{-3/2}\sqrt{\log m}L^{7/2}T^{5/3}\lambda^{-5/3}(1+\sqrt{T/\lambda})\Big]\bigg]\notag \\
     & \leq  \sqrt{\tilde d\log(1+TK/\lambda) +1 + C_1m^{-1/6}\sqrt{\log m} L^4T^{5/3}\lambda^{-1/6}}\bigg[\sqrt{1+C_1m^{-1/6}\sqrt{\log m}L^4T^{7/6}\lambda^{-7/6}}\notag \\
    &\qquad \cdot \bigg(\nu\sqrt{\tilde d\log(1+TK/\lambda) +1 + 2C_2m^{-1/6}\sqrt{\log m} L^4T^{5/3}\lambda^{-1/6}-2\log\delta} + \sqrt{\lambda}S\bigg)\notag \\
    &\qquad +  (\lambda +C_3TL)\Big[(1- \eta m \lambda)^{J/2} \sqrt{T/(m\lambda)} + m^{-3/2}\sqrt{\log m}L^{7/2}T^{5/3}\lambda^{-5/3}(1+\sqrt{T/\lambda})\Big]\bigg]\notag,
\end{align}
where the first inequality holds due to the fact that $\gamma_{t-1} \leq \gamma_T$, the second inequality holds due to \eqref{selfnormal:3.2} and \eqref{selfnormal:100}, the third inequality holds due to \eqref{selfnormal:3.1}. This completes our proof. 
\end{proof}

\section{Proofs of Technical Lemmas in Appendix \ref{section:new2}}\label{section:new3}

\subsection{Proof of Lemma \ref{lemma:ntkconverge}}
In this section we prove Lemma \ref{lemma:ntkconverge}, we need the following lemma from \citet{arora2019exact}:
\begin{lemma}[Theorem 3.1, \citet{arora2019exact}]\label{lemma:ntkconcen}
Fix $\epsilon>0$ and $\delta \in (0,1)$. Suppose that 
\begin{align}
    m  = \Omega\bigg(\frac{L^6\log (L/\delta)}{\epsilon^4}\bigg),\notag
\end{align}
then for any $i,j \in [TK]$, with probability at least $1-\delta$ over random initialization of $\btheta_0$, we have
\begin{align}
    |\la \gb(\xb^i; \btheta_0), \gb(\xb^j;\btheta_0)\ra/m - \Hb_{i,j}| \leq \epsilon.\label{lemma:ntkconcen_1}
\end{align}
\end{lemma}
\begin{proof}[Proof of Lemma \ref{lemma:ntkconverge}]
Taking union bound over $i,j \in [TK]$, we have that if 
\begin{align}
    m  = \Omega\bigg(\frac{L^6\log (T^2K^2L/\delta)}{\epsilon^4}\bigg),\notag
\end{align}
then with probability at least $1-\delta$, \eqref{lemma:ntkconcen_1} holds for all $(i,j) \in [TK] \times [TK]$. Therefore, we have
\begin{align}
    \|\Gb^\top\Gb - \Hb\|_F& = \sqrt{\sum_{i=1}^{TK}\sum_{j=1}^{TK} |\la \gb(\xb^i; \btheta_0), \gb(\xb^j;\btheta_0)\ra/m - \Hb_{i,j}|^2} \leq TK\epsilon.\notag
\end{align}

\end{proof}

\subsection{Proof of Lemma \ref{lemma:newboundreference}}
In this section we prove Lemma \ref{lemma:newboundreference}. 
During the proof, for simplicity, we omit the subscript $t$ by default. We define the following quantities:
\begin{align}
    &\Jb^{(j)} = 
    \Big(\gb(\xb_{1,a_1};\btheta^{(j)}),
    \dots,
    \gb(\xb_{t,a_t};\btheta^{(j)})\Big)
     \in \RR^{ (md+m^2(L-2)+m)\times t},\notag \\
    &\Hb^{(j)} = [\Jb^{(j)}]^\top\Jb^{(j)} \in \RR^{t \times t},\notag \\
    &\fb^{(j)} = (f(\xb_{1,a_1};\btheta^{(j)}),\dots,f(\xb_{t,a_t}; \btheta^{(j)}))^\top \in \RR^{t \times 1},\notag \\
    & \yb = (r_{1,a_1},\dots,r_{t,a_t}) \in \RR^{t \times 1}.\notag
\end{align}
Then the update rule of $\btheta^{(j)}$ can be written as follows:
\begin{align}
    \btheta^{(j+1)} &= \btheta^{(j)} - \eta \big[\Jb^{(j)}( \fb^{(j)} - \yb) + m\lambda(\btheta^{(j)} - \btheta^{(0)})\big].\label{lemma:in_the_ball_-1}
\end{align}
We also define the following auxiliary sequence $\{\tilde\btheta^{(k)}\}$ during the proof:
\begin{align}
     \tilde \btheta^{(0)} = \btheta^{(0)},\ 
     \tilde \btheta^{(j+1)} = \tilde \btheta^{(j)} - \eta\big[\Jb^{(0)}([\Jb^{(0)}]^\top(\tilde \btheta^{(j)} - \tilde \btheta^{(0)}) - \yb ) +  m\lambda(\tilde\btheta^{(j)} - \tilde\btheta^{(0)})\big].\notag
\end{align}

Next lemma provides perturbation bounds for $\Jb^{(j)}, \Hb^{(j)}$ and $\|\fb^{(j+1)} - \fb^{(j)} - [\Jb^{(j)}]^\top(\btheta^{(j+1)} - \btheta^{(j)})\|_2$.
\begin{lemma}\label{lemma:jacob}
There exist constants $\{\bar C_i\}_{i=1}^6>0$ such that for any $\delta > 0$, if $\tau$ satisfies that
\begin{align}
    \bar C_1m^{-3/2}L^{-3/2}[\log(TKL^2/\delta)]^{3/2}\leq\tau \leq \bar C_2 L^{-6}[\log m]^{-3/2},\notag
\end{align}
then with probability at least $1-\delta$,
if for any $j \in [J]$, $\|\btheta^{(j)} -\btheta^{(0)}\|_2\leq \tau$,  we have the following inequalities for any $j,s \in[J]$,
\begin{align}
     &\big\|\Jb^{(j)}\big\|_F \leq \bar C_4 \sqrt{tmL},\label{lemma:jacobcc1}\\
     &\|\Jb^{(j)} - \Jb^{(0)}\|_F  \leq \bar C_5\sqrt{tm\log m}\tau^{1/3}L^{7/2},\label{lemma:jacobcc2}\\
&\big\|\fb^{(s)} - \fb^{(j)} - [\Jb^{(j)}]^\top(\btheta^{(s)} - \btheta^{(j)})\big\|_2 \leq 
    \bar C_6 \tau^{4/3}L^3 \sqrt{tm \log m}, \label{lemma:jacobcc3} \\
    &\|\yb\|_2 \leq \sqrt{t}.\label{lemma:jacobcc4}
\end{align}
\end{lemma}

Next lemma gives an upper bound for $\|\fb^{(j)} - \yb\|_2$.
\begin{lemma}\label{lemma:linearconvergence}
There exist constants $\{\bar C_i\}_{i=1}^4>0$ such that for any $\delta > 0$, if $\tau, \eta$ satisfy that
    \begin{align}
    &\bar C_1m^{-3/2}L^{-3/2}[\log(TKL^2/\delta)]^{3/2}\leq\tau \leq \bar C_2 L^{-6}[\log m]^{-3/2},\notag ,\\
    &\eta \leq \bar C_3(m\lambda + tmL)^{-1},\notag \\
    &\tau^{8/3} \leq \bar C_4m(\lambda\eta)^2L^{-6}t^{-1}(\log m)^{-1},\notag
\end{align}
then with probability at least $1-\delta$,
if for any $j \in [J]$, $\|\btheta^{(j)}-\btheta^{(0)}\|_2\leq \tau$,  we have that for any $j \in [J]$, $\|\fb^{(j)} - \yb\|_2\leq 2\sqrt{t}$.

\end{lemma}

Next lemma gives an upper bound of the distance between auxiliary sequence $\|\tilde \btheta^{(j)} -  \btheta^{(0)}\|_2$.
\begin{lemma}\label{lemma:implicitbias}
There exist constants $\{\bar C_i\}_{i=1}^3>0$ such that for any $\delta \in (0,1)$, if $\tau, \eta$ satisfy that
\begin{align}
    &\bar C_1m^{-3/2}L^{-3/2}[\log(TKL^2/\delta)]^{3/2}\leq\tau \leq \bar C_2 L^{-6}[\log m]^{-3/2},\notag ,\\
    &\eta \leq \bar C_3(tmL + m\lambda)^{-1} ,\notag
\end{align}
then with probability at least $1-\delta$, we have that for any $j \in[J]$,
\begin{align}
    &\big\|\tilde\btheta^{(j)} - \btheta^{(0)}\big\|_2 \leq \sqrt{t/(m\lambda)},\notag \\ &\big\|\tilde\btheta^{(j)} - \btheta^{(0)} - \bar\Zb^{-1}\bar \bbb/\sqrt{m}\big\|_2 \leq (1 - \eta m \lambda)^{j/2} \sqrt{t/(m\lambda)}\notag
\end{align}
\end{lemma}
With above lemmas, we prove Lemma \ref{lemma:newboundreference} as follows.
\begin{proof}[Proof of Lemma \ref{lemma:newboundreference}]
Set $\tau = 2\sqrt{t/(m\lambda)}$. First we assume that $\|\btheta^{(j)}-\btheta^{(0)}\|_2\leq\tau$ for all $0 \leq j \leq J$. Then with this assumption and the choice of $m,\tau$, we have that Lemma \ref{lemma:jacob},  \ref{lemma:linearconvergence} and \ref{lemma:implicitbias} hold. Then we have
\begin{align}
\big\|\btheta^{(j+1)} - \tilde \btheta^{(j+1)}\big\|_2 &= \big\|\btheta^{(j)} - \tilde \btheta^{(j)} - \eta(\Jb^{(j)} - \Jb^{(0)})( \fb^{(j)} - \yb)-\eta m \lambda (\btheta^{(j)} - \tilde \btheta^{(j)}) \notag \\
& \qquad -  \eta \Jb^{(0)} ( \fb^{(j)} -  [\Jb^{(0)}]^\top(\tilde\btheta^{(j)} - \btheta^{(0)}))\big\|_2\notag \\
&= \Big\|(1-\eta m\lambda)(\btheta^{(j)} - \tilde \btheta^{(j)}) - \eta(\Jb^{(j)} - \Jb^{(0)})(  \fb^{(j)} - \yb) \notag \\
&\qquad -  \eta \Jb^{(0)} \Big[  \fb^{(j)}   -  [\Jb^{(0)}]^\top (\btheta^{(j)} - \btheta^{(0)})+  [\Jb^{(0)}]^\top(\btheta^{(j)} - \tilde\btheta^{(j)})\Big]\Big\|_2\notag \\
& \leq  \underbrace{\eta\big\|(\Jb^{(j)} - \Jb^{(0)})(  \fb^{(j)} - \yb)\big\|_2}_{I_1} + \underbrace{\eta\|\Jb^{(0)}\|_2\big\|  \fb^{(j)}  -  [\Jb^{(0)}] (\btheta^{(j)} - \btheta^{(0)})\big\|_2}_{I_2} \notag \\
&\qquad + \underbrace{\big\|\big[\Ib -\eta (m \lambda \Ib+ \Hb^{(0)})\big](\tilde\btheta^{(j)} - \btheta^{(j)})\big\|_2}_{I_3}, \label{lemma:in_the_ball_-53.10}
\end{align}
where the inequality holds due to triangle inequality.
We now bound $I_1, I_2$ and $I_3$ separately. For $I_1$, we have
\begin{align}
    I_1 &\leq \eta\big\|\Jb^{(j)} - \Jb^{(0)}\big\|_2\| \fb^{(j)}- \yb\|_2 \leq \eta C_2t\sqrt{m\log m}\tau^{1/3}L^{7/2} ,\label{lemma:in_the_ball_-53.999}
\end{align}
where $C_2>0$ is a constant, the first inequality holds due to the definition of matrix spectral norm and the second inequality holds due to \eqref{lemma:jacobcc2} in Lemma \ref{lemma:jacob} and Lemma \ref{lemma:linearconvergence}.
For $I_2$, we have
\begin{align}
    I_2 & \leq \eta\big\|\Jb^{(0)}\big\|_2 \Big\|\fb^{(j)}  - \Jb^{(0)}(\btheta^{(j)} - \btheta^{(0)})\Big\|_2 \leq \eta C_3 tmL^{7/2}\tau^{4/3}\sqrt{\log m},\label{lemma:in_the_ball_-53.11}
\end{align}
where $C_3>0$, the first inequality holds due to matrix spectral norm, the second inequality holds due to \eqref{lemma:jacobcc1} and \eqref{lemma:jacobcc3} in Lemma \ref{lemma:jacob} and the fact that $\fb^{(0)} = \zero$ by random initialization over $\btheta^{(0)}$. For $I_3$, we have
\begin{align}
    I_3 \leq \big\|\Ib -\eta (m \lambda \Ib+ \Hb^{(0)})\big\|_2\big\|\tilde\btheta^{(j)} - \btheta^{(j)}\big\|_2 \leq (1-\eta m\lambda)\big\|\tilde\btheta^{(j)} - \btheta^{(j)}\big\|_2,\label{lemma:in_the_ball_-53.111}
\end{align}
where the first inequality holds due to spectral norm inequality, the second inequality holds since 
\begin{align}
    \eta (m \lambda \Ib+ \Hb^{(0)}) = \eta(m \lambda \Ib+ [\Jb^{(0)}]^\top\Jb^{(0)}) \preceq \eta(m \lambda \Ib+ C_1tmL\Ib) \preceq \Ib,\notag
\end{align}
for some $C_1>0$, the first inequality holds due to  \eqref{lemma:jacobcc1} in Lemma \ref{lemma:jacob}, the second inequality holds due to the choice of $\eta$. 

Substituting \eqref{lemma:in_the_ball_-53.999}, \eqref{lemma:in_the_ball_-53.11} and \eqref{lemma:in_the_ball_-53.111} into \eqref{lemma:in_the_ball_-53.10}, we obtain
\begin{align}
    &\big\|\btheta^{(j+1)} - \tilde\btheta^{(j+1)}\big\|_2 \leq (1-\eta m \lambda)\big\|\btheta^{(j)} - \tilde\btheta^{(j)}\big\|_2 + C_4\big(\eta t\sqrt{m\log m}\tau^{1/3}L^{7/2} +\eta  tmL^{7/2}\tau^{4/3}\sqrt{\log m}\big),\label{lemma:in_the_ball_-53.104}
\end{align}
where $C_4>0$ is a constant. 
By recursively applying  \eqref{lemma:in_the_ball_-53.104} from $0$ to $j$, we have
\begin{align}
    \big\|\btheta^{(j+1)} - \tilde\btheta^{(j+1)}\big\|_2
    & \leq C_4\frac{\eta t\sqrt{m\log m}\tau^{1/3}L^{7/2} + \eta  tmL^{7/2}\tau^{4/3}\sqrt{\log m}}{ \eta m \lambda}\notag \\
    & =  C_5m^{-2/3}\sqrt{\log m}L^{7/2}t^{5/3}\lambda^{-5/3}(1+\sqrt{t/\lambda})\notag \\
    & \leq \frac{\tau}{2},\label{lemma:in_the_ball_-53.1041}
\end{align}
where $C_5>0$ is a constant, the equality holds by the definition of $\tau$, the last inequality holds due to the choice of $m$, where
\begin{align}
m^{1/6}\geq C_6\sqrt{\log m}L^{7/2}t^{7/6}\lambda^{-7/6}(1+\sqrt{t/\lambda}),\notag
\end{align}
and $C_6>0$ is a constant.
Thus, for any $j \in[J]$, we have
\begin{align}
    \|\btheta^{(j)} - \btheta^{(0)}\|_2 \leq \|\tilde\btheta^{(j)} - \btheta^{(0)}\|_2 + \|\btheta^{(j)}-\tilde\btheta^{(j)} \|_2 \leq \sqrt{t/(m\lambda)} + \tau/2 = \tau,\label{lemma:in_the_ball_-53.1055}
\end{align}
where the first inequality holds due to triangle inequality, the second inequality holds due to Lemma \ref{lemma:implicitbias}. \eqref{lemma:in_the_ball_-53.1055} suggests that our assumption $\|\btheta^{(j)} -\btheta^{(0)}\|_2\leq \tau$ holds for any $j$. Note that we have the following inequality by Lemma \ref{lemma:implicitbias}:
\begin{align}
    \big\|\tilde\btheta^{(j)} - \btheta^{(0)} - (\bar\Zb)^{-1}\bar \bbb/\sqrt{m}\big\|_2 \leq  (1- \eta m \lambda)^j \sqrt{t/(m\lambda)}.\label{lemma:in_the_ball_-53.105}
\end{align}
Using \eqref{lemma:in_the_ball_-53.1041} and \eqref{lemma:in_the_ball_-53.105}, we have
\begin{align}
    \big\|\btheta^{(j)} - \btheta^{(0)} - \bar\Zb^{-1}\bar \bbb/\sqrt{m}\big\|_2 \leq (1- \eta m \lambda)^{j/2} \sqrt{t/(m\lambda)} + C_5m^{-2/3}\sqrt{\log m}L^{7/2}t^{5/3}\lambda^{-5/3}(1+\sqrt{t/\lambda}).\notag
\end{align}
This completes the proof.
\end{proof}

\subsection{Proof of Lemma \ref{lemma:newboundz}}
In this section we prove Lemma \ref{lemma:newboundz}.

\begin{proof}[Proof of Lemma \ref{lemma:newboundz}]
Set $\tau = 2\sqrt{t/(m\lambda)}$. By Lemma \ref{lemma:newboundreference} we have that $\|\btheta_i - \btheta_0\|_2 \leq \tau$ for $i \in [t]$. 
$\|\Zb_t\|_2$ can be bounded as follows.
\begin{align}
    \|\Zb_t\|_2 &= \bigg\|\lambda \Ib + \sum_{i=1}^t\gb(\xb_{i,a_i}; \btheta_{i-1})\gb(\xb_{i,a_i}; \btheta_{i-1})^\top/m\bigg\|_2\notag \\
    & \leq \lambda +  \bigg\|\lambda \Ib + \sum_{i=1}^t\gb(\xb_{i,a_i}; \btheta_{i-1})\gb(\xb_{i,a_i}; \btheta_{i-1})^\top/m\bigg\|_2\notag \\
    & \leq \lambda + \sum_{i=1}^t\big\|\gb(\xb_{i,a_i}; \btheta_{i-1})\big\|_2^2/m \notag \\
    & \leq \lambda + C_0tL,\notag
\end{align}
where $C_0>0$ is a constant, the first inequality holds due to the fact that $\|\ab\ab^\top\|_F = \|\ab\|_2^2$, the second inequality holds due to Lemma \ref{lemma:cao_boundgradient} with the fact that $\|\btheta_i -\btheta_0\|_2\leq \tau$.
We bound $\|\Zb_t - \bar \Zb_t\|_2$ as follows. We have
\begin{align}
    \|\Zb_t - \bar \Zb_t\|_F &= \bigg\|\sum_{i=1}^t\Big(\gb(\xb_{i,a_i}; \btheta_{0})\gb(\xb_{i,a_i}; \btheta_{0})^\top -  \gb(\xb_{i,a_i}; \btheta_{i})\gb(\xb_{i,a_i}; \btheta_{i})^\top\Big)/m\bigg\|_F\notag \\
    & \leq \sum_{i=1}^t\Big\|\gb(\xb_{i,a_i}; \btheta_{0})\gb(\xb_{i,a_i}; \btheta_{0})^\top -  \gb(\xb_{i,a_i}; \btheta_{i})\gb(\xb_{i,a_i}; \btheta_{i})^\top\Big\|_F/m\notag \\
    & \leq \sum_{i=1}^t\Big(\big\|\gb(\xb_{i,a_i}; \btheta_{0})\big\|_2 + \big\|\gb(\xb_{i,a_i}; \btheta_{i})\big\|_2\Big)\big\|\gb(\xb_{i,a_i}; \btheta_{0}) - \gb(\xb_{i,a_i}; \btheta_{i})\big\|_2/m,\label{eq:newboundz_0}
\end{align}
where the first inequality holds due to triangle inequality, the second inequality holds the fact that $\|\ab\ab^\top - \bbb\bbb^\top\|_F \leq (\|\ab\|_2 + \|\bbb\|_2)\|\ab - \bbb\|_2$ for any vectors $\ab, \bbb$. To bound \eqref{eq:newboundz_0}, we have
\begin{align}
    \big\|\gb(\xb_{i,a_i}; \btheta_{0})\big\|_2, \big\|\gb(\xb_{i,a_i}; \btheta_{i})\big\|_2 \leq C_1\sqrt{mL},\label{eq:newboundz_1}
\end{align}
where $C_1>0$ is a constant, the inequality holds due to Lemma \ref{lemma:cao_boundgradient} with the fact that $\|\btheta_i -\btheta_0\|_2\leq \tau$. We also have
\begin{align}
    \big\|\gb(\xb_{i,a_i}; \btheta_{0}) - \gb(\xb_{i,a_i}; \btheta_{i})\big\|_2 \leq C_2\sqrt{\log m}\tau^{1/3}L^3\|\gb(\xb_j;  \btheta_0)\|_2 \leq C_3\sqrt{m\log m}\tau^{1/3}L^{7/2},\label{eq:newboundz_2}
\end{align}
where $C_2, C_3>0$ are constants, the first inequality holds due to Lemma \ref{lemma:cao_gradientdifference} with the fact that $\|\btheta_i -\btheta_0\|_2\leq \tau$, the second inequality holds due to Lemma \ref{lemma:cao_boundgradient}. Substituting \eqref{eq:newboundz_1} and \eqref{eq:newboundz_2} into \eqref{eq:newboundz_0}, we have
\begin{align}
    \|\Zb_t - \bar \Zb_t\|_F \leq C_4t\sqrt{\log m}\tau^{1/3}L^4,\notag
\end{align}
where $C_4>0$ is a constant. We now bound $\log\det \bar \Zb_t - \log\det \Zb_t$. 
It is easy to verify that $\bar \Zb_t = \lambda\Ib + \bar\Jb\bar\Jb^\top$, $\Zb_t = \lambda\Ib + \Jb\Jb^\top$, where
\begin{align}
&\bar \Jb =  \Big(\gb(\xb_{1,a_1};\btheta_0),
    \dots,
    \gb(\xb_{t,a_t};\btheta_0)\Big)/\sqrt{m},\notag \\
    &\Jb = \Big(\gb(\xb_{1,a_1};\btheta_0),
    \dots,
    \gb(\xb_{t,a_t};\btheta_{t-1})\Big)/\sqrt{m}.\notag
\end{align}
We have the following inequalities: 
\begin{align}
    \log \frac{\det(\bar \Zb_t)}{\det(\lambda\Ib)} - \log\frac{ \det(\Zb_t)}{\det (\lambda\Ib)}& = \log \det(\Ib + \bar\Jb\bar\Jb^\top/\lambda) - \log \det(\Ib + \Jb\Jb^\top/\lambda)\notag \\
    & = \log \det(\Ib + \bar\Jb^\top\bar\Jb/\lambda) - \log \det(\Ib + \Jb^\top\Jb/\lambda)\notag \\
      &\leq  \la (\Ib + \Jb^\top\Jb/\lambda)^{-1}, \bar\Jb^\top\bar\Jb - \Jb^\top\Jb\ra \notag \\
     & \leq \|(\Ib + \Jb^\top\Jb/\lambda)^{-1}\|_F\|\bar\Jb^\top\bar\Jb - \Jb^\top\Jb\|_F\notag \\
     & \leq \sqrt{t}\|(\Ib + \Jb^\top\Jb/\lambda)^{-1}\|_2\|\bar\Jb^\top\bar\Jb - \Jb^\top\Jb\|_F\notag \\
     & \leq \sqrt{t}\|\bar\Jb^\top\bar\Jb - \Jb^\top\Jb\|_F,\label{eq:newboundz_4}
\end{align}
where the second equality holds due to the fact that $\det (\Ib + \Ab\Ab^\top) = \det (\Ib + \Ab^\top\Ab)$, the first inequality holds due to the fact that $\log\det$ function is convex, the second inequality hold due to the fact that $\la\Ab, \Bb\ra \leq \|\Ab\|_F\|\Bb\|_F$, the third inequality holds since $\Ib + \Jb^\top\Jb/\lambda$ is a $t$-dimension matrix, the fourth inequality holds since $\Ib + \Jb^\top\Jb/\lambda \succeq \Ib$. We have
\begin{align}
    &\|\bar\Jb^\top\bar\Jb - \Jb^\top\Jb\|_F \notag \\
    & \leq t\max_{1 \leq i,j \leq t}
    \Big|\gb(\xb_{i,a_i}; \btheta_{0})^\top\gb(\xb_{j,a_j}; \btheta_{0}) -  \gb(\xb_{i,a_i}; \btheta_{i})^\top\gb(\xb_{j,a_j}; \btheta_{j})\Big|/m\notag \\
    & \leq t\max_{1 \leq i,j \leq t}
    \big\|\gb(\xb_{i,a_i}; \btheta_{0}) - \gb(\xb_{i,a_i}; \btheta_{i})\big\|_2\big\|\gb(\xb_{j,a_j}; \btheta_{j})\big\|_2/m \notag \\
    &\qquad + \big\|\gb(\xb_{j,a_j}; \btheta_{0}) -  \gb(\xb_{j,a_j}; \btheta_{j})\big\|_2\big\|\gb(\xb_{i,a_i}; \btheta_{0})\big\|_2/m\notag\\
    & \leq C_5t\sqrt{\log m} \tau^{1/3}L^4,\label{eq:newboundz_5}
\end{align}
where $C_5>0$ is a constant, the first inequality holds due to the fact that $\|\Ab\|_F \leq t\max|\Ab_{i,j}|$ for any $\Ab \in \RR^{t \times t}$, the second inequality holds due to the fact $|\ab^\top\ab' - \bbb^\top\bbb'| \leq \|\ab-\bbb\|_2\|\bbb'\|_2 +\|\ab'-\bbb'\|_2\|\ab\|_2$, the third inequality holds due to \eqref{eq:newboundz_1} and \eqref{eq:newboundz_2}. Substituting \eqref{eq:newboundz_5} into \eqref{eq:newboundz_4}, we obtain 
\begin{align}
    \log \frac{\det(\bar\Zb_t)}{\det(\lambda\Ib)} - \log\frac{ \det( \Zb_t)}{\det (\lambda\Ib)} \leq C_5t^{3/2}\sqrt{\log m}\tau^{1/3}L^4.\notag
\end{align}
Using the same method, we also have
\begin{align}
    \log \frac{\det(\Zb_t)}{\det(\lambda\Ib)} - \log\frac{ \det(\bar \Zb_t)}{\det (\lambda\Ib)} \leq C_5t^{3/2}\sqrt{\log m}\tau^{1/3}L^4.\notag
\end{align}
This completes our proof.

\end{proof}

\section{Proofs of Lemmas in Appendix \ref{section:new3}}

\subsection{Proof of Lemma \ref{lemma:jacob}}
In this section we give the proof of Lemma \ref{lemma:jacob}.
\begin{proof}[Proof of Lemma \ref{lemma:jacob}]
It can be verified that $\tau$ satisfies the conditions of Lemmas \ref{lemma:cao_functionvalue}, \ref{lemma:cao_gradientdifference} and \ref{lemma:cao_boundgradient}. Thus, Lemmas \ref{lemma:cao_functionvalue}, \ref{lemma:cao_gradientdifference} and \ref{lemma:cao_boundgradient} hold.
We will show that for any $j \in [J]$, the following inequalities hold. 
First, we have
\begin{align}
    \big\|\Jb^{(j)}\big\|_F \leq \sqrt{t}\max_{i \in [t]}\big\|\gb(\xb_{i,a_i}; \btheta^{(j)})\big\|_2 \leq C_1 \sqrt{tmL},\label{lemma:jacob_-1}
\end{align}
where $C_1>0$ is a constant, the first inequality holds due to the fact that $\|\Jb^{(j)}\|_F \leq \sqrt{t}\|\Jb^{(j)}\|_{2, \infty}$, the second inequality holds due to Lemma \ref{lemma:cao_boundgradient}.

We also have
\begin{align}
    \|\Jb^{(j)} - \Jb^{(0)}\|_F &\leq C_2\sqrt{\log m}\tau^{1/3}L^3\|\Jb^{(0)}\|_F \leq C_3\sqrt{tm\log m}\tau^{1/3}L^{7/2},\label{lemma:jacob_0}
\end{align}
where $C_2, C_3>0$ are constants, the first inequality holds due to Lemma \ref{lemma:cao_gradientdifference} with the assumption that $\|\btheta^{(j)}-\btheta^{(0)}\|_2\leq \tau$, the second inequality holds due to \eqref{lemma:jacob_-1}. 

We also have 
\begin{align}
    &\big\|\fb^{(s)} - \fb^{(j)} - [\Jb^{(j)}]^\top(\btheta^{(s)} - \btheta^{(j)})\big\|_2 \notag \\
    & \leq \max_{i \in [t]}\sqrt{t}\big|f(\xb_{i,a_i};\btheta^{(s)}) - f(\xb_{i,a_i};\btheta^{(j)}) - \la \gb(\xb_{i,a_i};\btheta^{(j)}), \btheta^{(s)} - \btheta^{(j)}\ra\big|\notag \\
    &\leq 
    C_4 \tau^{4/3}L^3 \sqrt{tm \log m},\notag
\end{align}
where $C_4>0$ is a constant, the first inequality holds due to the the fact that $\|\xb \|_2 \leq \sqrt{t}\max|x_i|$ for any $\xb \in \RR^t$, the second inequality holds due to Lemma \ref{lemma:cao_functionvalue} with the assumption that $\|\btheta^{(j)} - \btheta^{(0)}\|_2 \leq \tau,  \|\btheta^{(s)} - \btheta^{(0)}\|_2 \leq \tau$. 

For $\|\yb\|_2$, we have $\|\yb\|_2 \leq \sqrt{t}\max_{1 \leq i \leq t}|r(\xb_{i, a_i})| \leq \sqrt{t}$. This completes our proof.

\end{proof}

\subsection{Proof of Lemma \ref{lemma:linearconvergence}}
\begin{proof}[Proof of Lemma \ref{lemma:linearconvergence}]
It can be verified that $\tau$ satisfies the conditions of Lemma \ref{lemma:jacob}, thus Lemma \ref{lemma:jacob} holds.
Recall that the loss function $L$ is defined as
\begin{align}
    L(\btheta) = \frac{1}{2}\|\fb(\btheta) - \yb\|_2^2 + \frac{m\lambda}{2}\|\btheta - \btheta^{(0)}\|_2^2.\notag
\end{align}
We define $\Jb(\btheta)$ and $\fb(\btheta)$ as follows:
\begin{align}
    &\Jb(\btheta) = 
    \Big(\gb(\xb_{1,a_1};\btheta),
    \dots,
    \gb(\xb_{t,a_t};\btheta)\Big)
     \in \RR^{ (md+m^2(L-2)+m)\times t},\notag \\
    &\fb(\btheta) = (f(\xb_{1,a_1};\btheta),\dots,f(\xb_{t,a_t}; \btheta))^\top \in \RR^{t \times 1}.\notag 
\end{align}
Suppose $\|\btheta -\btheta^{(0)}\|_2\leq \tau$.
Then by the fact that $\|\cdot\|_2^2/2$ is $1$-strongly convex and $1$-smooth, we have the following inequalities: 
\begin{align}
    &L(\btheta') - L(\btheta) \notag \\
    &\leq \la \fb(\btheta) - \yb, \fb(\btheta')-  \fb(\btheta) \ra + \frac{1}{2}\big\|\fb(\btheta')-  \fb(\btheta)\big\|_2^2 + m\lambda \la \btheta - \btheta^{(0)}, \btheta' - \btheta\ra + \frac{m\lambda}{2}\big\|\btheta' - \btheta\big\|_2^2\notag \\
    & = \la \fb(\btheta) - \yb, [\Jb(\btheta)]^\top(\btheta' - \btheta)  + \eb\ra + \frac{1}{2}\big\|[\Jb(\btheta)]^\top(\btheta' - \btheta) + \eb\big\|_2^2 \notag \\
    &\qquad + m\lambda \la \btheta - \btheta^{(0)}, \btheta' - \btheta\ra + \frac{m\lambda}{2}\big\|\btheta' - \btheta\big\|_2^2\notag \\
    & =  \la \Jb(\btheta)(\fb(\btheta) - \yb) + m\lambda (\btheta - \btheta^{(0)}),\btheta' - \btheta \ra + \la  \fb(\btheta) - \yb, \eb\ra \notag \\
    &\qquad + \frac{1}{2}\big\|[\Jb(\btheta)]^\top(\btheta' - \btheta) + \eb\big\|_2^2 + \frac{m\lambda}{2}\big\|\btheta' - \btheta\big\|_2^2\notag \\
    & = \la \nabla L(\btheta), \btheta' - \btheta\ra + \underbrace{\la  \fb(\btheta) - \yb, \eb\ra + \frac{1}{2}\big\|[\Jb(\btheta)]^\top(\btheta' - \btheta) + \eb\big\|_2^2 + \frac{m\lambda}{2}\big\|\btheta' - \btheta\big\|_2^2}_{I_1},\label{eq:linearconvergence_0}
\end{align}
where $\eb = \fb(\btheta') - \fb(\btheta) - \Jb(\btheta)^\top(\btheta' - \btheta)$. 
$I_1$ can be bounded as follows:
\begin{align}
    I_1 &\leq \|\fb(\btheta) - \yb\|_2\|\eb\|_2 + \|\Jb(\btheta)\|_2^2 \|\btheta' - \btheta\|_2^2 + \|\eb\|_2^2 + \frac{m\lambda}{2}\big\|\btheta' - \btheta\big\|_2^2\notag \\
    & \leq \frac{C_1}{2}\bigg( (m\lambda + tmL)\big\|\btheta' - \btheta\big\|_2^2\bigg) + \|\fb(\btheta) - \yb\|_2\|\eb\|_2 + \|\eb\|_2^2,\label{eq:linearconvergence_1}
\end{align}
where the first inequality holds due to Cauchy-Schwarz inequality, the second inequality holds due to the fact that $\|\Jb(\btheta)\|_2 \leq C_2\sqrt{tmL}$ with $\|\btheta - \btheta^{(0)}\|_2\leq \tau$ by \eqref{lemma:jacobcc1} in Lemma \ref{lemma:jacob}. Substituting \eqref{eq:linearconvergence_1} into \eqref{eq:linearconvergence_0}, we obtain
\begin{align}
    L(\btheta') - L(\btheta) \leq \la\nabla L(\btheta), \btheta' - \btheta\ra + \frac{C_1}{2}\bigg( (m\lambda + tmL)\big\|\btheta' - \btheta\big\|_2^2\bigg) +\|\fb(\btheta) - \yb\|_2\|\eb\|_2 + \|\eb\|_2^2. \label{eq:linearconvergence_2}
\end{align}
Taking $\btheta' = \btheta - \eta\nabla L(\btheta)$, then by \eqref{eq:linearconvergence_2}, we have
\begin{align}
    L\big(\btheta - \eta \nabla L(\btheta)\big) - L(\btheta) \leq -\eta\|\nabla L(\btheta)\|_2^2\big[1 - C_1(m\lambda + tmL)\eta\big] + \|\fb(\btheta) - \yb\|_2\|\eb\|_2 + \|\eb\|_2^2.\label{eq:linearconvergence_2.5}
\end{align}
By the $1$-strongly convexity of $\|\cdot\|_2^2$, we further have
\begin{align}
    &L(\btheta') - L(\btheta)\notag \\
    & \geq \la \fb(\btheta) - \yb, \fb(\btheta')-  \fb(\btheta) \ra + m\lambda \la \btheta - \btheta^{(0)}, \btheta' - \btheta\ra + \frac{m\lambda}{2}\big\|\btheta' - \btheta\big\|_2^2\notag \\
    & = \la \fb(\btheta) - \yb, [\Jb(\btheta)]^\top(\btheta' - \btheta)  + \eb\ra  + m\lambda \la \btheta - \btheta^{(0)}, \btheta' - \btheta\ra + \frac{m\lambda}{2}\big\|\btheta' - \btheta\big\|_2^2\notag \\
    & = \la \nabla L(\btheta), \btheta' - \btheta\ra + \frac{m\lambda}{2}\big\|\btheta' - \btheta\big\|_2^2 + \la \fb(\btheta) - \yb, \eb\ra\notag \\
    & \geq \la\nabla L(\btheta), \btheta' - \btheta\ra + \frac{m\lambda}{2}\big\|\btheta' - \btheta\big\|_2^2 - \|\fb(\btheta) - \yb\|_2\|\eb\|_2\notag \\
    & \geq -\frac{\|\nabla L(\btheta)\|_2^2}{2m\lambda}- \|\fb(\btheta) - \yb\|_2\|\eb\|_2,\label{eq:linearconvergence_3}
\end{align}
where the second inequality holds due to Cauchy-Schwarz inequality, the last inequality holds due to the fact that $\la\ab, \xb\ra + c \|\xb\|_2^2 \geq -\|\ab\|_2^2/(4c)$ for any vectors $\ab, \xb$ and $c>0$. Substituting \eqref{eq:linearconvergence_3} into \eqref{eq:linearconvergence_2.5}, we obtain 
\begin{align}
    &L\big(\btheta - \eta \nabla L(\btheta)\big) - L(\btheta)\notag \\
    & \leq 2m\lambda\eta(1 - C_1(m\lambda + tmL)\eta)\big[L(\btheta') - L(\btheta) +\|\fb(\btheta) - \yb\|_2\|\eb\|_2 \big] + \|\fb(\btheta) - \yb\|_2\|\eb\|_2 + \|\eb\|_2^2\notag \\
    & \leq m\lambda\eta \big[L(\btheta') - L(\btheta) +\|\fb(\btheta) - \yb\|_2\|\eb\|_2 \big] + \|\fb(\btheta) - \yb\|_2\|\eb\|_2 + \|\eb\|_2^2\notag \\
    & \leq m\lambda\eta \big[L(\btheta') - L(\btheta) +\|\fb(\btheta) - \yb\|_2^2/8 + 2\|\eb\|_2^2 \big] + m\lambda\eta\|\fb(\btheta) - \yb\|_2^2/8 + 2\|\eb\|_2^2/(m\lambda\eta) + \|\eb\|_2^2\notag \\
    & \leq m\lambda\eta\big(L(\btheta') - L(\btheta)/2\big) + \|\eb\|_2^2\big(1 + 2 m \lambda\eta + 2/(m \lambda\eta)\big),\label{eq:linearconvergence_4}
\end{align}
where the second inequality holds due to the choice of $\eta$, 
third inequality holds due to Young's inequality, fourth inequality holds due to the fact that $\|\fb(\btheta) - \yb\|_2^2 \leq 2L(\btheta)$. Now taking $\btheta = \btheta^{(j)}$ and $\btheta' = \btheta^{(0)}$, rearranging \eqref{eq:linearconvergence_4}, with the fact that $\btheta^{(j+1)} = \btheta^{(j)} - \eta \nabla L(\btheta^{(j)})$, we have 
\begin{align}
    &L(\btheta^{(j+1)}) - L(\btheta^{(0)})\notag \\
    & \leq (1-m\lambda\eta/2)\big[L(\btheta^{(j)}) - L(\btheta^{(0)})\big] + m\lambda\eta/2L(\btheta^{(0)}) + \|\eb\|_2^2\big(1 + 2 m \lambda\eta + 2/(m \lambda\eta)\big)\notag \\
    & \leq (1-m\lambda\eta/2)\big[L(\btheta^{(j)}) - L(\btheta^{(0)})\big] +  m\lambda\eta/2\cdot t + m\lambda\eta/2\cdot t\notag\\
    & \leq (1-m\lambda\eta/2)\big[L(\btheta^{(j)}) - L(\btheta^{(0)})\big] + m\lambda\eta t,\label{eq:linearconvergence_5}
\end{align}
where the second inequality holds due to the fact that $L(\btheta^{(0)}) = \|\fb(\btheta^{(0)}) - \yb\|_2^2/2 = \|\yb\|_2^2/2 \leq t$, and 
\begin{align}
    \big(1 + 2 m \lambda\eta + 2/(m \lambda\eta)\big)\|\eb\|_2^2 \leq 3/(m\lambda \eta)\cdot  C_2 \tau^{8/3}L^6 tm\log m \leq tm\lambda\eta/2,
\end{align}
where the first inequality holds due to \eqref{lemma:jacobcc3} in Lemma \ref{lemma:jacob}, the second inequality holds due to the choice of $\tau$. 
Recursively applying  \eqref{eq:linearconvergence_5} for $u$ times, we have
\begin{align}
    L(\btheta^{(j+1)}) - L(\btheta^{(0)}) \leq \frac{m\lambda\eta t}{m\lambda\eta/2} = 2t,\notag
\end{align}
which implies that $\|\fb^{(j+1)} - \yb\|_2 \leq 2\sqrt{t}$.
This completes our proof. 
\end{proof}

\subsection{Proof of Lemma \ref{lemma:implicitbias}}
In this section we prove Lemma \ref{lemma:implicitbias}.
\begin{proof}[Proof of Lemma \ref{lemma:implicitbias}]
It can be verified that $\tau$ satisfies the conditions of Lemma \ref{lemma:jacob}, thus Lemma \ref{lemma:jacob} holds. It is worth noting that $\tilde\btheta^{(j)}$ is the sequence generated by applying gradient descent on the following problem:
\begin{align}
    \min_{\btheta}\tilde \cL(\btheta) = \frac{1}{2}\|[\Jb^{(0)}]^\top(\btheta - \btheta^{(0)}) - \yb\|_2^2 + \frac{m\lambda}{2}\big\|\btheta - \btheta^{(0)}\big\|_2^2.\notag
\end{align}
Then $\|\btheta^{(0)} - \tilde\btheta^{(j)}\|_2$ can be bounded as
\begin{align}
    \frac{m\lambda}{2}\|\btheta^{(0)} - \tilde\btheta^{(j)}\|_2^2 &\leq \frac{1}{2}\|[\Jb^{(0)}]^\top(\tilde\btheta^{(j)} - \btheta^{(0)}) - \yb\|_2^2 + \frac{m\lambda}{2}\big\|\tilde\btheta^{(j)} - \btheta^{(0)}\big\|_2^2\notag \\
    & \leq \frac{1}{2}\|[\Jb^{(0)}]^\top(\tilde\btheta^{(0)} - \btheta^{(0)}) - \yb\|_2^2 + \frac{m\lambda}{2}\big\|\tilde\btheta^{(0)} - \btheta^{(0)}\big\|_2^2\notag \\
    & \leq t/2,\notag
\end{align}
where the first inequality holds trivially, the second inequality holds due to the monotonic decreasing property brought by gradient descent, the third inequality holds due to \eqref{lemma:jacobcc4} in Lemma \ref{lemma:jacob}.
It is easy to verify that $\tilde \cL$ is a $m\lambda$-strongly convex and function and $C_1(tmL + m\lambda)$-smooth function, since 
\begin{align}
    \nabla^2 \tilde \cL \preceq \big(\big\|\Jb^{(0)}\big\|_2^2 + m\lambda\big)\Ib \preceq C_1(tmL + m\lambda),\notag
\end{align}
where the first inequality holds due to the definition of $\tilde \cL$, the second inequality holds due to \eqref{lemma:jacobcc1} in Lemma \ref{lemma:jacob}. Since we choose $\eta \leq C_2(tmL + m\lambda)^{-1}$ for some small enough $C_2>0$, then by standard results of gradient descent on ridge linear regression, $\tilde\btheta^{(j)}$ converges to $\btheta^{(0)} + (\bar\Zb)^{-1}\bar\bbb/\sqrt{m}$ with the convergence rate
\begin{align}
    \big\|\tilde\btheta^{(j)} - \btheta^{(0)} - \bar\Zb^{-1}\bbb/\sqrt{m}\big\|_2^2 &\leq (1 - \eta m \lambda)^j\cdot \frac{2}{m\lambda}(\cL(\btheta^{(0)}) - \cL\big(\btheta^{(0)} + \bar\Zb^{-1}\bbb/\sqrt{m} \big))\notag \\
    &\leq \frac{2(1 - \eta m \lambda)^j}{m\lambda}\cL(\btheta^{(0)})\notag \\
    & =\frac{2(1 - \eta m \lambda)^j}{m\lambda}\cdot \frac{\|\yb\|_2^2}{2}\notag \\
    & \leq (1 - \eta m \lambda)^jt,\notag
\end{align}  
where the first inequality holds due to the convergence result for gradient descent and the fact that $\btheta^{(0)} + (\bar\Zb)^{-1}\bar\bbb/\sqrt{m}$ is the minimal solution to $\cL$, the second inequality holds since $\cL \geq 0$, the last inequality holds due to Lemma \ref{lemma:jacob}.

\end{proof}

\section{A Variant of $\algname$}\label{sec:original}
In this section, we present a variant of $\algname$ called $\algname_0$. Compared with Algorithm~\ref{algorithm:2}, 
The main differences between \algname\ and $\algname_0$ are as follows:
\algname\ uses gradient descent to train a deep neural network to learn the reward function $h(\xb)$ based on observed contexts and rewards.  In contrast, $\algname_0$ uses matrix inversions to obtain parameters in closed forms.
At each round, \algname\ uses the current DNN parameters ($\btheta_t$) to compute an upper confidence bound.  In contrast, $\algname_0$ computes the UCB using the initial parameters ($\btheta_0$).

\begin{algorithm}
	\caption{$\algname_0$}\label{alg:old}
	\begin{algorithmic}[1]
	\STATE \textbf{Input:} number of rounds $T$, regularization parameter $\lambda$, exploration parameter $\nu$, confidence parameter $\delta$, norm parameter $S$, network width $m$, network depth $L$
	\STATE \textbf{Initialization:} Generate each entry of $\Wb_l$ independently from $N(0, 2/m)$ for $1 \leq l \leq L-1$, and each entry of $\Wb_L$ independently from $N(0, 1/m)$.  Define $\bphi(\xb)  = \gb(\xb; \btheta_0)/\sqrt{m}$, where
 $\btheta_0 = [\text{vec}(\Wb_1)^\top,\dots,\text{vec}(\Wb_L)^\top]^\top \in \RR^{p}$
	\STATE $\Zb_0 =  \lambda\Ib, \,\, \bbb_0 = \zero$
    \FOR{$t = 1,\dots,T$}
    \STATE Observe $\{\xb_{t,a}\}_{a=1}^K$ and compute
    \begin{align}
        (a_t, \tilde \btheta_{t,a_t}) = \argmax_{a \in [K], \btheta \in \cC_{t-1}} \la \bphi(\xb_{t,a}), \btheta - \btheta_0\ra\label{alg:1}
    \end{align}
    \STATE Play $a_t$ and receive reward $r_{t,a_t}$
    \STATE Compute
    \begin{align}
        \Zb_t = \Zb_{t-1} + \bphi(\xb_{t,a_t})\bphi(\xb_{t,a_t})^\top \in \RR^{p\times p},\ \,\,
        \bbb_t = \bbb_{t-1} + r_{t,a_t} \bphi(\xb_{t,a_t}) \in \RR^{p}\notag
    \end{align}
    \STATE Compute $\btheta_t = \Zb_t^{-1}\bbb_t  + \btheta_0\in \RR^{p}$
    \STATE Construct $\cC_{t}$ as
    \begin{align}
        \cC_{t} = \{\btheta:\|\btheta_t - \btheta\|_{\Zb_t} \leq \gamma_{t} \}, \quad\text{where}\quad
        \gamma_t = \nu\sqrt{\log\frac{\det \Zb_t}{\det \lambda \Ib}-2\log\delta} + \sqrt{\lambda}S\label{alg:4}
    \end{align}
        \ENDFOR
	\end{algorithmic}
\end{algorithm}

\end{document}